\documentclass[11pt]{article}
\usepackage[preprint]{acl}

\usepackage{iftex}
\ifPDFTeX
  \usepackage[T1]{fontenc}
  \usepackage[utf8]{inputenc}
  \usepackage{newtxtext,newtxmath}    
  \usepackage{inconsolata}            
\else
  \usepackage{fontspec}
  \usepackage{newtxmath}              
\fi

\usepackage{microtype}                
\usepackage{graphicx}
\graphicspath{{../}}

\PassOptionsToPackage{table,dvipsnames}{xcolor}
\usepackage{xcolor}
\usepackage{booktabs}
\usepackage{array}
\usepackage{colortbl}
\usepackage{enumitem}

\usepackage{amsmath}

\usepackage{amssymb,mathtools}
 
\usepackage{amsthm}

\theoremstyle{plain}
\newtheorem{theorem}{Theorem}[section]
\newtheorem{lemma}[theorem]{Lemma}
\newtheorem{definition}[theorem]{Definition}

\newtheorem{proposition}[theorem]{Proposition}

\newtheorem{assumption}[theorem]{Assumption}

\newtheorem{fact}[theorem]{Fact}

\usepackage{algorithm}
\usepackage{algpseudocode}
\usepackage{subfig}                   
\usepackage{wrapfig}

\usepackage{url}

\usepackage{lineno}

\usepackage{xspace} 
\makeatletter
\DeclareRobustCommand{\eqfit}[1]{\resizebox{\linewidth}{!}{$\displaystyle #1$}}
\makeatother

\makeatletter
\newcommand*{\RN}[1]{\expandafter\@slowromancap\romannumeral #1@}
\makeatother

\usepackage{listings}

\definecolor{codegreen}{HTML}{478058}
\definecolor{codegray}{rgb}{0.5,0.5,0.5}
\definecolor{codepurple}{HTML}{4F5E80}
\definecolor{backcolour}{rgb}{0.95,0.95,0.95}

\lstdefinestyle{mystyle}{
  backgroundcolor=\color{backcolour},
  commentstyle=\color{codegreen},
  keywordstyle=\color{magenta},
  numberstyle=\tiny\color{codegray},
  stringstyle=\color{codepurple},
  basicstyle=\ttfamily\scriptsize,
  breaklines=true,
  keepspaces=true,
  showstringspaces=false,
  tabsize=2
}
\usepackage{caption}
\usepackage{tcolorbox}
\tcbuselibrary{listings,breakable,skins}

\definecolor{mybluetitle}{HTML}{4B527E}
\definecolor{SFTcolor}{HTML}{F0F0F0}
\definecolor{REDIcolor}{HTML}{E6FFE6}
\definecolor{NeutralColor}{HTML}{FFFFFF}

\newtcolorbox{promptbox}[2][]{
  enhanced, breakable, center title,
  left*=0pt, right*=0pt,
  boxsep=2pt, left=5pt, right=5pt,
  colback=backcolour,
  fonttitle=\bfseries\rmfamily,
  fontupper=\scriptsize,
  title={\footnotesize\strut #2},
  #1
}
\newtcolorbox{insightbox}[2][]{
  enhanced, breakable,
  skin first=enhanced, skin middle=enhanced, skin last=enhanced,
  left*=0pt, right*=0pt,
  boxsep=2pt, left=4pt, right=4pt, top=4pt, bottom=4pt,
  colback=mybluetitle!10!white,
  colframe=mybluetitle!90!white,
  coltitle=white,
  fonttitle=\bfseries,
  title={#2},
  #1
}

\newcommand{\wh}{\widehat}
\newcommand{\wt}{\widetilde}
\newcommand{\ov}{\overline}
\newcommand{\un}{\underline}

\newcommand{\R}{\mathbb{R}}

\newcommand{\loss}{{\mathcal{L}}}

\DeclareMathOperator*{\E}{\mathbb{E}}

\makeatletter
\def\onedot{\ifx\@let@token.\else.\null\fi\xspace}
\makeatother

\newcommand{\xpar}[1]{\noindent\textbf{#1}\hspace{0.5em}}
\usepackage{xurl}

\title{Harnessing Negative Signals:\\Reinforcement Distillation from Teacher Data for LLM Reasoning} 

\author{
  \vspace{-5pt}\\
  \textbf{Shuyao Xu$^{1,2,}$\thanks{Work done during an internship at INF AI (INFLY TECH (Shanghai) Co., Ltd).}\quad Cheng Peng$^{2}$\quad Jiangxuan Long$^{2,*}$\quad Weidi Xu$^2$\quad Wei Chu$^{2}$\quad Yuan Qi$^{2}$}\vspace{3pt} \\
  $^1$National University of Singapore\quad\quad$^2$INF AI\vspace{3pt} \\
    \texttt{\small shuyao@u.nus.edu, wdxu@inftech.ai}\vspace{8pt}  \\
  Code:~\, \url{https://github.com/Tim-Siu/reinforcement-distillation} \\
  \vspace{-5pt}
}

\begin{document}

\maketitle
\begin{abstract}
Recent advances in model distillation show that data from advanced reasoning models can effectively train smaller student models. However, standard practices discard incorrect reasoning traces—valuable, yet underutilized data. This paper addresses the critical question: \textit{How can both positive and negative distilled reasoning traces be effectively leveraged to maximize LLM reasoning performance in an offline setting?} We employ a two-stage training recipe: first, Supervised Fine-Tuning (SFT) on positive traces, followed by a refinement stage using both positive and negative traces. We find that a simple REINFORCE-style objective, which we term the Reinforcement Distillation (REDI) objective, outperforms established preference optimization methods like DPO and SimPO in this distillation context. Our empirical evaluations demonstrate the effectiveness of this approach. Notably, our \texttt{Qwen-REDI-1.5B} model, trained on just 131k traces from the open Open-R1 dataset, achieves an 83.1\% score on MATH-500. Its performance matches that of \texttt{DeepSeek-R1-Distill-Qwen-1.5B}, a model trained on 800k proprietary data. This result showcases the remarkable data efficiency of utilizing previously discarded negative traces.
\end{abstract}
\section{Introduction}
Recent breakthroughs with large reasoning models, such as DeepSeek-R1 and OpenAI's o1, have demonstrated remarkable capabilities in complex reasoning tasks \citep{deepseekai2025deepseekr1incentivizingreasoningcapability, openai2024openaio1card}. Techniques like test-time scaling facilitate longer Chain-of-Thought (CoT) processes and induce sophisticated reasoning behaviors, enhancing model performance in domains like mathematics. For base models initially lacking such advanced reasoning, two primary methods are employed to cultivate these abilities. The first, large-scale reinforcement learning (RL), directly applies RL algorithms to the base model, continually optimizing it through online exploration \citep{deepseekai2025deepseekr1incentivizingreasoningcapability,tinyzero,zeng2025simplerlzooinvestigatingtamingzero}. However, RL approaches typically demand strong base models to achieve their full potential and are computationally intensive \citep{yue2025limit-of-rlvr, deepseekai2025deepseekr1incentivizingreasoningcapability}. In contrast, distillation—learning from reasoning traces (e.g., CoT) generated by large ``teacher`` models—emerges as an attractive alternative for smaller, more efficient student models. This approach offers a practical and cost-effective pathway to extend their reasoning capabilities \citep{openthoughts,deepseekai2025deepseekr1incentivizingreasoningcapability}. Benefiting from open datasets distilled from powerful reasoning models like DeepSeek-R1 \citep{openthoughts,bespoke_stratos,openr1}, openly post-trained models have shown strong performance \citep{openthoughts,ye2025limoreasoning,muennighoff2025s1simpletesttimescaling,wen2025lightr1curriculumsftdpo}, although a performance gap remains compared to their closed-data counterparts.
\begin{figure}[t]
\begin{center}
\includegraphics[width=0.98\linewidth]{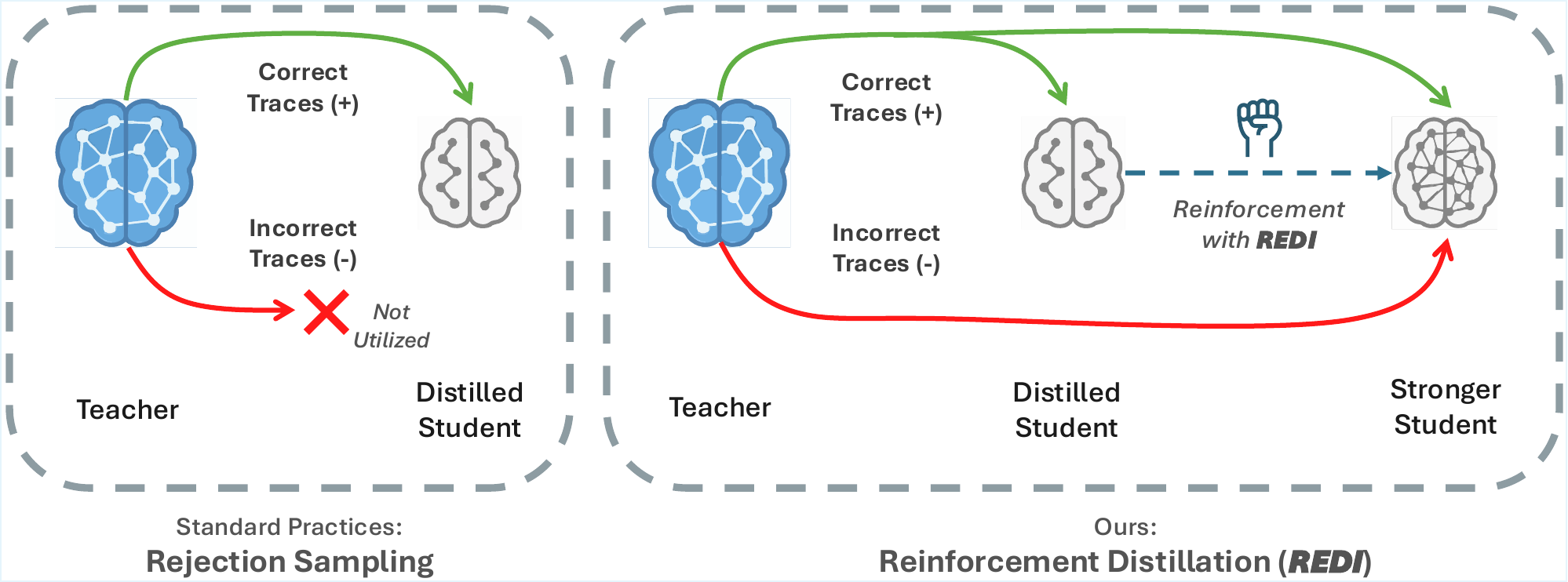}
\end{center}
\caption{\small
\textbf{Standard distillation practices via Rejection Sampling \textit{vs.} our proposed Reinforcement Distillation (REDI).}
Our REDI recipe can utilize previously discarded incorrect reasoning traces generated by the teacher and yield stronger distilled models.
}
\label{fig:intro}
\vspace{-10pt}
\end{figure}

\begin{table*}[t]
    \centering
    \caption{\small \textbf{Model Performance Comparison (pass@1 over 16 samples) across reasoning benchmarks.} Our \texttt{Qwen-REDI-1.5B}, trained with the REDI recipe on just 131k open data points, achieves the highest average score. This performance surpasses \texttt{DeepSeek-R1-Distill-Qwen-1.5B} (trained on 800k proprietary data) \citep{deepseekai2025deepseekr1incentivizingreasoningcapability}, demonstrating REDI's \textbf{remarkable data efficiency}. REDI enhances reasoning by effectively utilizing both positive and negative distilled examples. Values in \textbf{bold} indicate the best performance in each column. \textsuperscript{*}Officially reported pass@1 results.}
        \label{tab:pass1_5ep}
    \resizebox{\textwidth}{!}{
    \begin{tabular}{lcccccc}
        \toprule
        Model & MATH-500 & AIME24 & AMC23 & Minerva & OlympiadBench & Avg. \\
        \midrule
        \rowcolor{NeutralColor} Qwen2.5-Math-1.5B-Instruct\textsuperscript{*} & 75.8 & 10.0 & 60.0 & \textbf{29.4} & 38.1 & 42.7 \\
        \rowcolor{NeutralColor} DeepSeek-R1-Distill-Qwen-1.5B & \textbf{83.2 \tiny{$\pm$0.2}} & \textbf{28.3 \tiny{$\pm$1.0}} & 62.1 \tiny{$\pm$0.8} & 26.0 \tiny{$\pm$0.4} & 43.1 \tiny{$\pm$0.3} & 48.6 \tiny{$\pm$0.3} \\
        \rowcolor{SFTcolor} Qwen-SFT-1.5B-5ep (SFT Baseline) & 80.4 \tiny{$\pm$0.3} & 21.9 \tiny{$\pm$1.1} & 57.5 \tiny{$\pm$0.6} & 27.5 \tiny{$\pm$0.4} & 41.5 \tiny{$\pm$0.3} & 45.8 \tiny{$\pm$0.3} \\
        \rowcolor{REDIcolor} \textbf{Qwen-REDI-1.5B} & 83.1 \tiny{$\pm$0.3} & 28.1 \tiny{$\pm$1.1} & \textbf{62.4 \tiny{$\pm$0.6}} & 28.8 \tiny{$\pm$0.3} & \textbf{45.2 \tiny{$\pm$0.2}} & \textbf{49.5 \tiny{$\pm$0.3}} \\
        \bottomrule
    \end{tabular}}
    \end{table*}

However, current distillation methodologies predominantly rely on rejection sampling, which involves leveraging only positive\footnote{We use ``positive" interchangeably with ``correct", and ``negative" interchangeably with ``incorrect".} reasoning traces---those whose final answers are verified. This practice means that negative traces, despite the significant computational effort invested in their generation, are typically underutilized. We hypothesize that these negative traces contain vital insights into common pitfalls and nuanced errors from which smaller models could learn, thereby further unlocking the potential of distilled data. This leads to the central research question we address:

\textit{How can we effectively leverage both positive and negative distilled reasoning traces to maximize LLM reasoning performance with a fixed distilled open dataset?}

To address this challenge, we first investigate the application of established preference optimization methods, such as Direct Preference Optimization (DPO) \citep{rafailov2024directpreferenceoptimizationlanguage} and SimPO \citep{meng2024simpo}, to this offline distillation setting. Our analysis reveals a critical \textbf{performance-stability trade-off}: while the Kullback-Leibler (KL) divergence penalty $\beta$ is essential for stable training, it simultaneously constrains the model's peak achievable performance\footnote{Performance refers to the test-time accuracy of the best checkpoint in a training run.}. This discovery motivated a deeper analysis, where we found that in the $\beta \to 0$ limit, these complex objectives converge to a simpler, reference-free and REINFORCE-style objective. This insight forms the basis of our approach. We propose \textbf{Re}inforcement \textbf{Di}stillation (REDI), which adopts this performant but unstable objective and introduces a \textbf{asymmetric weighting} scheme. By down-weighting the gradient from negative traces, REDI restores training stability without sacrificing the performance benefits, offering a simple and effective method for learning from both positive and negative signals.

Our key contributions are:
\begin{enumerate}[topsep=1pt,itemsep=2pt,leftmargin=20pt]
    \item We provide the first systematic study on the utilization of both correct and incorrect distilled reasoning traces. We identify and analyze a \textbf{performance-stability trade-off} inherent in the KL regularization of methods like DPO, demonstrating that it limits peak performance in this practical setting.
    \item Motivated by this analysis, we propose the \textbf{Reinforcement Distillation (REDI) objective}, a simple, asymmetrically weighted, and REINFORCE-style loss function. REDI is designed to capture the high-performance potential of the $\beta \to 0$ limit of DPO while mitigating the associated training instability, providing a simpler and more effective alternative.
    \item We empirically demonstrate that our two-stage recipe of SFT combined with REDI training consistently outperforms both Rejection Sampling SFT and SFT combined with DPO/SimPO. Our \texttt{Qwen-REDI-1.5B} model achieves performance comparable to models trained on much larger proprietary datasets, showcasing the data efficiency of our recipe.
\end{enumerate}

The remainder of this paper is organized as follows: Section~\ref{sec:methodology} details the Reinforcement Distillation methodology. Section~\ref{sec:setup} describes the experimental setup. Section~\ref{sec:results} presents our results and analysis. Section~\ref{sec:related_work} discusses related work, and Section~\ref{sec:conclusion} concludes the paper.
\section{Methodology}
\label{sec:methodology}

\subsection{Problem Setting and Data}
\label{subsec:problem_setting_data}

We operate in an offline distillation setting with a fixed dataset collected via a common distillation pipeline. The dataset originates from a set of problems, each denoted by $x$. For each problem $x$, a capable ``teacher" model is employed to generate reasoning traces. The generation process for a specific problem $x$ continues until a correct reasoning trace, $y_w$, is successfully produced. During these attempts, incorrect traces, $y_l$, might also be generated before $y_w$ is obtained.

From these generated traces, we construct two distinct datasets for our two-stage training recipe:
\begin{enumerate}[topsep=1pt, itemsep=1pt, leftmargin=20pt]
    \item \textbf{Positive Traces Dataset ($\mathcal{D}_{\text{SFT}}$):} This dataset comprises all pairs $(x, y_w)$, where $y_w$ is a correct reasoning trace generated by the teacher for problem $x$.
    \item \textbf{Preference Pairs Dataset ($\mathcal{D}_{\text{Pref}}$):} This dataset is constructed from the subset of problems $x$ for which at least one incorrect trace was generated before the correct trace $y_w$ was obtained. For each such problem $x$, we form a preference tuple $(x, y_w, y_l)$ by pairing its correct trace $y_w$ with \textit{one} selected incorrect trace $y_l$ generated for the same problem. This selection strategy is adopted for simplicity and aligns with observations from datasets like Open-R1 \citep{openr1}, where most problems that have negative examples feature only one such instance.
\end{enumerate}

Our overall objective is to train a student LLM, $\pi_\theta$, to maximize its reasoning performance by effectively leveraging all information within the pre-collected $\mathcal{D}_{\text{SFT}}$ and $\mathcal{D}_{\text{Pref}}$ datasets.

\subsection{The Reinforcement Distillation Recipe}

\subsubsection{Stage 1: Supervised Fine-Tuning (SFT) on Positive Traces}
The first stage involves standard Supervised Fine-Tuning (SFT) of the base LLM on the $\mathcal{D}_{\text{SFT}}$ dataset, which contains only positive (correct) reasoning traces $(x, y_w)$. The SFT objective is to maximize the likelihood of generating the correct trace $y_w$ given the problem $x$:
\begin{equation}
    \loss_{\text{SFT}}(\theta) = - \E_{(x, y_w) \sim \mathcal{D}_{\text{SFT}}} \left[ \log \pi_\theta(y_w|x) \right].
    \label{eq:sft_loss}
\end{equation}
This initial SFT stage serves several key purposes. Primarily, it adapts the base model to the specific style and format of the reasoning traces. Furthermore, it provides a strong initial policy, denoted as $\pi_\text{SFT}$, which can subsequently serve as a reference for methods like DPO or as the starting point for our REDI objective in the second stage. Finally, this stage establishes a baseline performance comparable to traditional SFT-only pipelines (i.e., training solely on positive examples), allowing us to quantify the gains achieved by later incorporating negative examples.

\subsubsection{Stage 2: Reinforcement with Positive and Negative Traces}
The second stage aims to further refine the model obtained from Stage 1 by leveraging the negative signals encoded in $\mathcal{D}_{\text{Pref}}$, which contains pairs of positive ($y_w$) and negative ($y_l$) traces for the same problem $x$.

\paragraph{Preliminary study.}
To contextualize our REDI objective, we first briefly review established preference optimization methods such as DPO \citep{rafailov2024directpreferenceoptimizationlanguage} and SimPO \citep{meng2024simpo}.

DPO optimizes the policy $\pi_\theta$ to align with human or model preferences while regularizing its deviation from a reference policy $\pi_\text{ref}$ (typically $\pi_\text{SFT}$ from Stage 1). Its loss function is:
\begin{equation}
    \eqfit{\loss_{\text{DPO}}(\theta; \pi_\text{ref}) = - \E_{(x, y_w, y_l) \,\sim\, \mathcal{D}_{\text{Pref}}}\!\left[\, \log \sigma \Big( \beta \big( \log \tfrac{\pi_\theta(y_w\mid x)}{\pi_\text{ref}(y_w\mid x)} - \log \tfrac{\pi_\theta(y_l\mid x)}{\pi_\text{ref}(y_l\mid x)} \big) \Big) \right]}
    \label{eq:dpo_loss}
\end{equation}
where $\sigma(\cdot)$ is the sigmoid function. The hyperparameter $\beta$ controls the strength of an implicit KL divergence penalty against $\pi_\text{ref}$, where larger $\beta$ values imply stronger regularization.

SimPO offers a reference-free alternative that incorporates sequence length normalization and an explicit margin $\gamma$:
\begin{equation}
    \eqfit{\loss_{\text{SimPO}}(\theta) = - \E_{(x, y_w, y_l) \,\sim\, \mathcal{D}_{\text{Pref}}}\!\left[\, \log \sigma \Big( \beta \big( \tfrac{\log \pi_\theta(y_w\mid x)}{|y_w|} - \tfrac{\log \pi_\theta(y_l\mid x)}{|y_l|} \big) - \gamma \Big) \right]}
    \label{eq:simpo_loss}
\end{equation}
Here, $|y|$ denotes the length (e.g., number of tokens) of sequence $y$. Similarly, in SimPO, higher values of $\beta$ act as a regularizer, leading to more stable training.

As empirically demonstrated in Section~\ref{sec:dpo_tradeoff}, while stronger regularization (e.g., higher $\beta$ in DPO or SimPO) can enhance training stability and permit larger gradient steps, it often results in lower peak model performance.

\paragraph{Towards a regularization-free objective.}
The observed trade-off between performance and stability associated with $\beta$ in methods like DPO and SimPO motivates exploring objectives that minimize or eliminate such explicit regularization. As detailed in Appendix~\ref{sec:appendix_simpo_relationship_derivation}, considering the $\beta \to 0$ limit of preference optimization objectives like SimPO yields the following simplified, regularization-free and REINFORCE-style \citep{reinforce} loss function (to be minimized):
\begin{equation}
    \eqfit{\loss_{\text{symm}}(\theta) = \E_{(x, y_w, y_l)\sim \mathcal{D}_{\text{Pref}}} \left( - \tfrac{\log \pi_\theta(y_{w}\mid x)}{|y_{w}|} + \tfrac{\log \pi_\theta(y_{l}\mid x)}{|y_{l}|} \right)}
    \label{eq:symm_loss}
\end{equation}
As empirically demonstrated in Section~\ref{subsec:redi_experiments}, this symmetric, reference-free objective (Eq.~\eqref{eq:symm_loss}) can achieve performance comparable to meticulously tuned DPO or SimPO, offering reduced hyperparameter tuning. Nevertheless, the tension between performance and stability persists: careful learning rate tuning remains crucial, as larger learning rates, while potentially accelerating learning and improving transient performance, often lead to early training collapse.

\paragraph{The REDI objective: asymmetric weighting for stability and performance.}
During experiments with DPO, SimPO, and the symmetric objective $\loss_{\text{symm}}$ (Section~\ref{sec:results}), we observed frequent early training collapses when learning rates were inadequately tuned. Collapse manifests as a rapid decrease in the likelihood of both positive ($y_w$) and negative ($y_l$) responses, accompanied by declining task accuracy. Recent studies attribute this instability to unintended side effects of off-policy gradients \citep{yan20253dpropertiesidentifyingchallengesdpo, razin2025unintentional, ren2025learning}. Specifically, gradient updates penalizing negative responses may inadvertently suppress semantically similar \textit{positive} responses, leading to degenerate solutions. Heuristic mitigations include auxiliary SFT losses or asymmetric $\beta$ tuning \citep{pang2024iterative,yan20253dpropertiesidentifyingchallengesdpo}.

Inspired by these insights, we propose \textbf{asymmetric weighting} for the simplified objective (Eq.~\eqref{eq:symm_loss}). By down-weighting gradients from negative traces, we preserve stability while maximizing peak performance.

The \textbf{REDI objective}, central to the second stage of our recipe, refines the model using an asymmetrically weighted, REINFORCE-style loss. The REDI loss to be minimized is defined as:
\begin{equation}
    \eqfit{\loss_{\text{REDI}}(\theta) = \E_{(x, y_w, y_l)\sim \mathcal{D}_{\text{Pref}}} \left[ -\tfrac{\log \pi_\theta(y_w\mid x)}{|y_w|} + \alpha \cdot \tfrac{\log \pi_\theta(y_l\mid x)}{|y_l|} \right]}
    \label{eq:redi_loss_expected}
\end{equation}
where $\alpha \in [0, 1]$ controls the penalty strength for negative traces:
\begin{itemize}[topsep=1pt,itemsep=2pt,leftmargin=20pt]
    \item $\alpha = 0$: Reduces to SFT on positive traces (ignores negatives).
    \item $\alpha = 1$: Recovers the symmetric objective (Eq.~\eqref{eq:symm_loss}).
\end{itemize}

The REDI objective, when optimized using gradient descent with an appropriate learning rate schedule (such as the one in Section~\ref{sec:setup}), is amenable to standard convergence analysis. Under typical $L$-smoothness assumptions for the loss function, this optimization process is guaranteed to converge to a stationary point. Further details and a formal proof are provided in Appendix~\ref{sec:appendix_convergence}. The asymmetric weighting ($\alpha < 1$) moderates gradient contributions from positive and negative samples, preventing collapse while maintaining aggressive learning dynamics.
\section{Experimental Setup}
\label{sec:setup}

\subsection{Data Curation}
\label{subsec:base_model_datasets}

Following the data pipeline described in Section~\ref{subsec:problem_setting_data}, we derived two datasets from the OpenR1-Math-Raw corpus \citep{openr1}; the \texttt{cn\_k12} subset was excluded due to its lower relative difficulty. The OpenR1-Math-Raw corpus provides two labels for correctness: one from the Llama judge and one from Math-Verify \citep{Kydlicek_Math-Verify_Math_Verification}. A response was considered correct if both labels were ``True"; otherwise, it was considered incorrect. More details are discussed in Appendix~\ref{subsec:appendix_data}.

The two datasets were constructed as follows:
\begin{itemize}[topsep=1pt,itemsep=1pt,leftmargin=15pt]
\item \textbf{Positive Traces Dataset ($\mathcal{D}_{\text{SFT}}$):} This dataset contains 78k problem-solution pairs $(x, y_w)$, where $y_w$ represents a correct reasoning trace. It was used for SFT in Stage 1.
\item \textbf{Preference Pairs Dataset ($\mathcal{D}_{\text{Pref}}$):} This dataset consists of 53k triplets $(x, y_w, y_l)$, where $y_w$ is a correct trace and $y_l$ is an incorrect trace for the same problem $x$. It was utilized in Stage 2.
\end{itemize}

\subsection{Training Configuration}
\label{subsec:sft_staged_training}
\paragraph{Stage 1 Configuration:}
In the first stage, we establish strong SFT baselines by fine-tuning the base \texttt{Qwen2.5-Math-1.5B} model on the $\mathcal{D}_{\text{SFT}}$ dataset. Two SFT baselines were prepared:
\begin{itemize}[topsep=1pt,itemsep=1pt,leftmargin=15pt]
    \item \texttt{Qwen-SFT-1.5B-3ep}: This model was trained for 3 epochs on $\mathcal{D}_{\text{SFT}}$. It served as the initial checkpoint for our comparative studies involving DPO, SimPO, and various REDI configurations.
    \item \texttt{Qwen-SFT-1.5B-5ep}: Observing continued SFT performance improvement beyond 3 epochs (as shown in Section~\ref{sec:sft_results}), this model was trained for 5 epochs on $\mathcal{D}_{\text{SFT}}$. This stronger SFT variant was used as the starting point for training our final \texttt{Qwen-REDI-1.5B} model.
\end{itemize}
For this SFT stage, all models were trained using the AdamW optimizer \citep{loshchilov2019decoupledweightdecayregularization} with a batch size of 128. The learning rate schedule featured a linear warmup for the initial 10\% of total training steps, followed by a linear decay to zero.
\paragraph{Stage 2 Configuration:}
The second stage involves further refining the SFT-tuned models using the $\mathcal{D}_{\text{Pref}}$ dataset, which contains preference pairs $(x, y_w, y_l)$. We applied the DPO, SimPO, and our proposed REDI objectives to the SFT checkpoints from Stage 1. All preference tuning methods were trained for 1 epoch over the $\mathcal{D}_{\text{Pref}}$ dataset. Similar to Stage 1, the AdamW optimizer and the same learning rate schedule (10\% warmup, then linear decay) were used. The batch size for this stage was 32. Specific hyperparameter settings for DPO (e.g., $\beta$ values, learning rates), SimPO (e.g., $\beta, \gamma$ values, learning rates), and REDI (e.g., $\alpha$ values, learning rates) were carefully tuned, with detailed ranges and chosen values provided in Appendix~\ref{subsec:appendix_hyperparam}.

\subsection{Evaluation Protocol}
\label{subsec:evaluation}
During all evaluations, generated samples were decoded using a temperature of $0.6$, Top P sampling with $p=0.95$, and a maximum generation length of $32,768$ tokens.

\xpar{Protocols:} We utilized two distinct configurations for evaluating model performance:
\begin{itemize}[topsep=1pt,itemsep=1pt,leftmargin=15pt]
    \item \textit{Intermediate Evaluations:} These evaluations, used for hyperparameter tuning, performance plotting, and ablation studies, were conducted using LightEval \citep{lighteval} on the MATH-500 benchmark \citep{lightman2023lets}. Performance was measured as pass@1, averaged over 4 generated samples per problem.
    \item \textit{Final Model Evaluations:} These evaluations, presented in comparison tables (e.g., Table~\ref{tab:pass1_5ep}), were performed using the DeepScaleR/rllm \citep{deepscaler2025} codebase on the mathematics benchmarks MATH-500, AIME24, AMC23, Minerva and out-of-distribution STEM benchmark OlympiadBench \citep{lewkowycz2022solvingquantitativereasoningproblems,he-etal-2024-olympiadbench}. Performance was measured as pass@1 (averaged over 16 generated samples) per problem for Tables~\ref{tab:pass1_5ep} and~\ref{tab:pass1_3ep}, and pass@16 for discussions in Section~\ref{sec:redi_online_rl_potential}.
\end{itemize}

\xpar{Reporting and SEM Calculation:}
The pass@k scores represent the proportion of problems for which at least one of $k$ generated samples is correct. When reporting pass@1 for main results, we also include standard error of the mean (SEM). See Appendix~\ref{subsec:appendix_sem} for the calculation of SEM.
\section{Results and Analysis}
\label{sec:results}
\subsection{Performance Limits of SFT-Only Training}
\label{sec:sft_results}
\begin{figure}[t]
        \centering
    \includegraphics[width=0.7\linewidth]{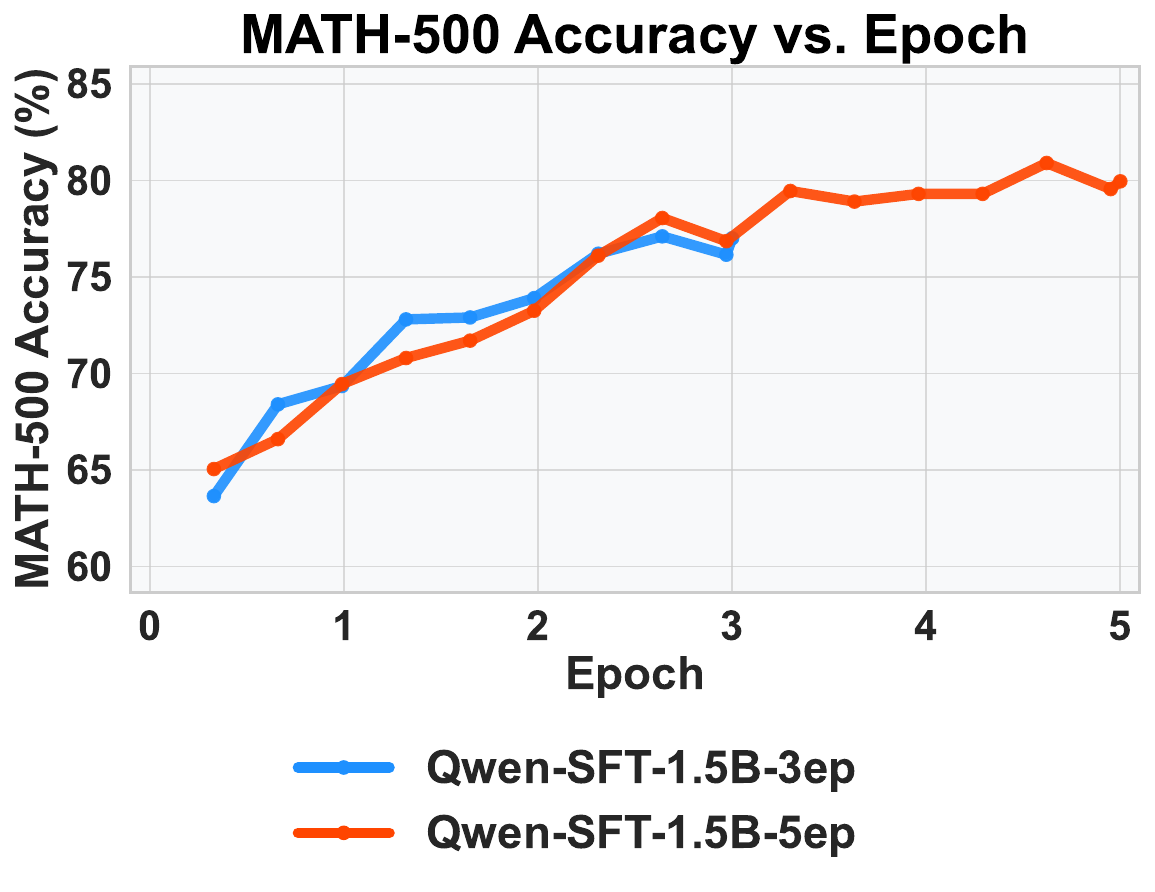}
        \caption{\small{SFT MATH-500 accuracy vs. training epochs.}}
    \label{fig:sft_math_accuracy_epoch}
    \end{figure}
We first establish the performance achievable using only positive distilled data via Supervised Fine-Tuning (SFT). As illustrated by Figure~\ref{fig:sft_math_accuracy_epoch}, performance increases for approximately 5 epochs before eventually plateauing. This observation highlights the limitations of learning solely from positive traces and motivates the utilization of negative signals.

\subsection{Performance-Stability Tradeoff in DPO}
\label{sec:dpo_tradeoff}

\begin{figure*}[ht]
    \centering
    \includegraphics[width=1.0\textwidth]{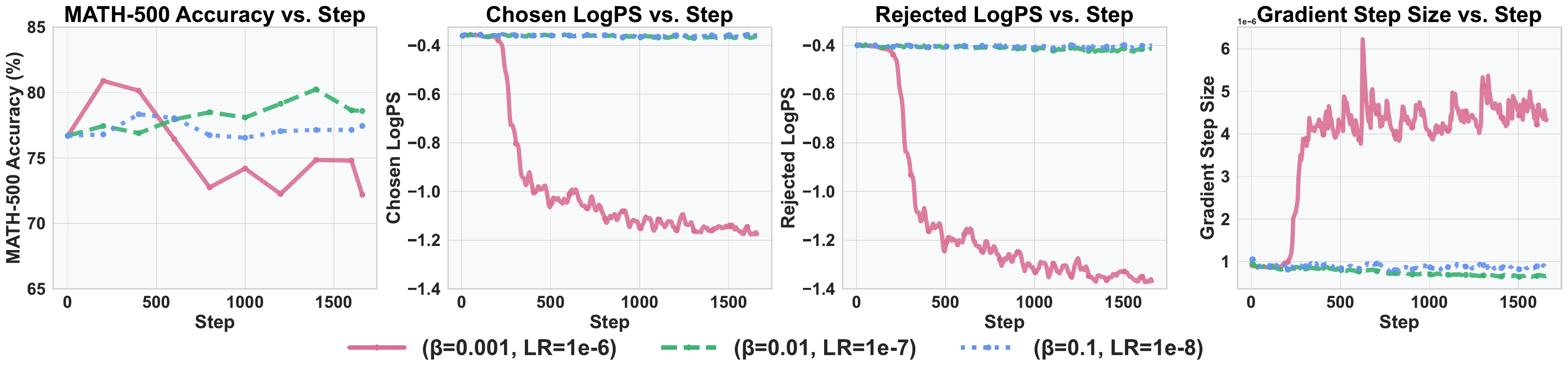}
    \caption{\small \textbf{DPO training dynamics with respect to $\beta$}, when initial gradient step sizes are controlled to be similar. LogPS visualizes the average per-token log probability of the model generating the chosen or rejected response. Gradient Step Size refers to the norm of the parameter update.}
    \label{fig:dpo_dynamics}
    \end{figure*}

\xpar{DPO dynamics with varying $\beta$ and similar initial gradient step sizes.}
Figure~\ref{fig:dpo_dynamics} illustrates DPO training dynamics for three configurations: $(\beta=0.001, \text{LR}=1 \times 10^{-6})$, $(\beta=0.01, \text{LR}=1 \times 10^{-7})$, and $(\beta=0.1, \text{LR}=1 \times 10^{-8})$. The learning rates were selected such that the initial gradient step sizes were comparable across these runs, as indicated in the ``Gradient Step Size vs. Step" subplot. The subsequent dynamics revealed a trade-off:
\begin{itemize}[topsep=1pt,itemsep=2pt,leftmargin=20pt]
    \item The lowest $\beta$ setting ($0.001$) achieved the highest peak accuracy (approximately 80.9\% on MATH-500) but subsequently experienced training collapse. This collapse in accuracy was accompanied by a sharp drop in chosen and rejected LogPS and a surge in gradient step size.
    \item Higher $\beta$ values ($0.01$, $0.1$) maintained stability throughout training but achieved lower peak accuracies (approximately 80.3\% and 78.3\%, respectively).
\end{itemize}
This exploration suggests that when initial gradient step sizes are matched, stronger KL regularization (higher $\beta$) yields more stable training, but performance can be constrained.

\xpar{Optimizing learning rates for different $\beta$ values.}
To further investigate whether the performance ceiling observed with higher $\beta$ values is an inherent limitation, we conducted learning rate (LR) sweeps for fixed $\beta$ values of $0.001$ and $0.01$ (Figure~\ref{fig:dpo_lr_sweeps}). This allows for a fairer comparison, as stronger regularization (higher $\beta$) can often accommodate larger gradient steps.
\begin{itemize}[topsep=1pt,itemsep=2pt,leftmargin=20pt]
    \item For $\beta=0.001$, an $\text{LR}$ of $2 \times 10^{-7}$ yielded the best peak performance at step 1000, reaching approximately 82.3\% on MATH-500.
    \item For $\beta=0.01$, an $\text{LR}$ of $2 \times 10^{-7}$ achieved the best peak for this $\beta$ value at step 1600, at approximately 81.2\%.
\end{itemize}

\begin{figure}[t]
    \centering
    \includegraphics[width=\linewidth]{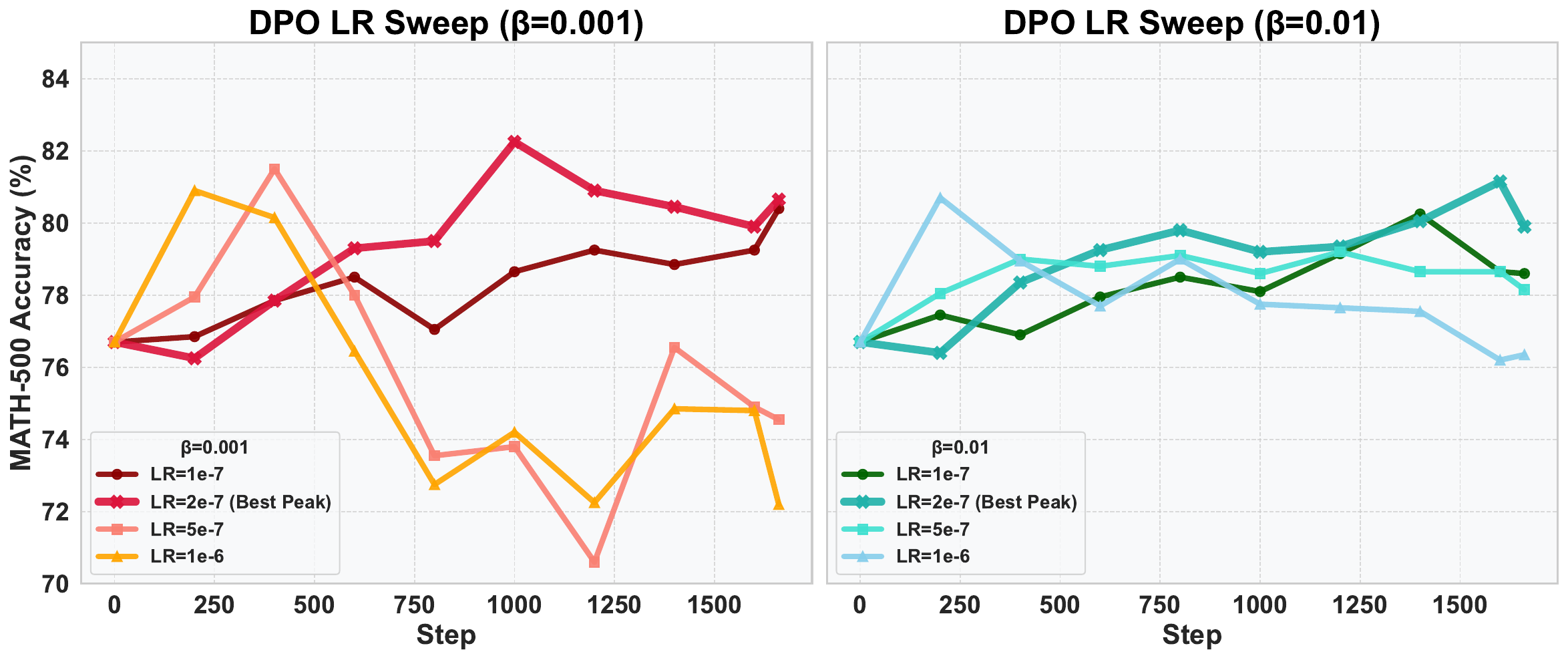}
    \caption{\small DPO MATH-500 accuracy with learning rate sweeps for $\beta=0.001$  and $\beta=0.01$.}
    \label{fig:dpo_lr_sweeps}
\end{figure}
Comparing the best-tuned runs from Figure~\ref{fig:dpo_lr_sweeps}, the configuration with the lower $\beta=0.001$ still achieved a significantly higher peak accuracy.

\begin{insightbox}{Observation: Stability and Peak Performance Trade-off in DPO with KL Regularization}
DPO's $\beta$ parameter, which controls KL regularization, presents a critical trade-off. Higher $\beta$ values enhance training stability, often allowing for more aggressive learning rates and more stable training steps. However, our experiments suggest that even with tuned LRs, \textbf{higher $\beta$ may restrict peak performance}. Conversely, lower $\beta$ values can yield higher peak performance, although they may require more careful tuning to avoid instability. 
\end{insightbox}

\begin{figure*}[ht]
    \centering
    \includegraphics[width=1.0\textwidth]{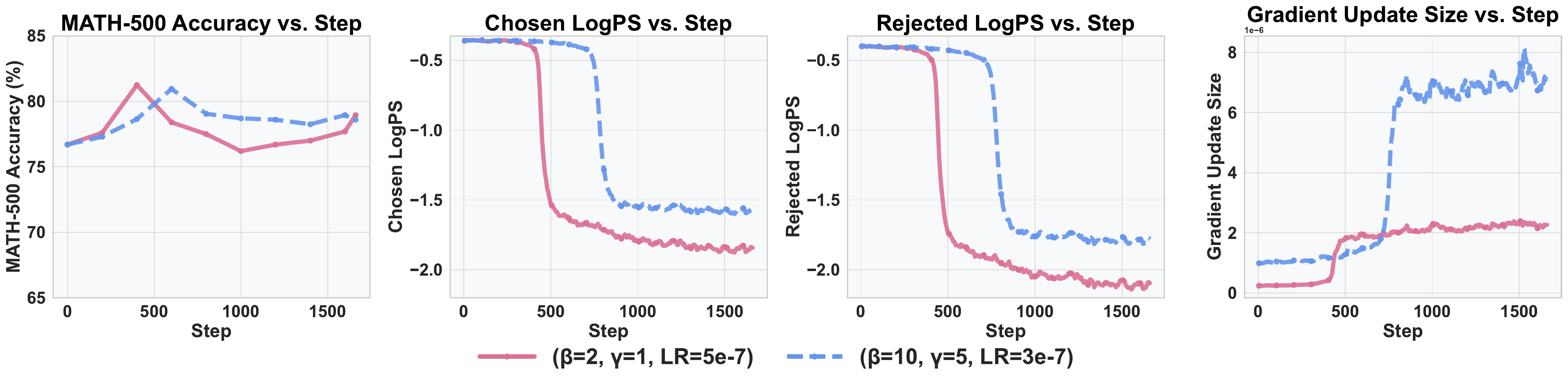}
\vspace{-15pt}
    \caption{\small SimPO training dynamics.}
    \label{fig:simpo_dynamics}
\vspace{-7pt}
\end{figure*}

\paragraph{Similar performance-stability tradeoff observed for SimPO.}
Preliminary experiments were also conducted with SimPO (Figure~\ref{fig:simpo_dynamics}). We found that with a fixed $\gamma/\beta$ ratio (0.5 in our tests), higher $\beta$ values correspond to stronger regularization effects. We experimented with $(\beta=2, \gamma=1, \text{LR}=5 \times 10^{-7})$ and $(\beta=10, \gamma=5, \text{LR}=3 \times 10^{-7})$. The $\beta=10$ run had a larger initial gradient update size and demonstrated greater stability (i.e., it ``collapsed" later than the $\beta=2$ run). However, its peak performance on MATH-500 was slightly lower than that of the $\beta=2$ run before its collapse. This reinforces the observation of a trade-off between stability and attainable peak performance.
\subsection{REDI: Achieving Stability and Performance with Asymmetric Weighting}
\label{subsec:redi_experiments}
Our REDI method directly optimizes log-likelihoods without KL regularization against a reference model, relying instead on asymmetric weighting to manage stability.

\begin{figure*}[ht]
    \centering
    \includegraphics[width=1.0\textwidth]{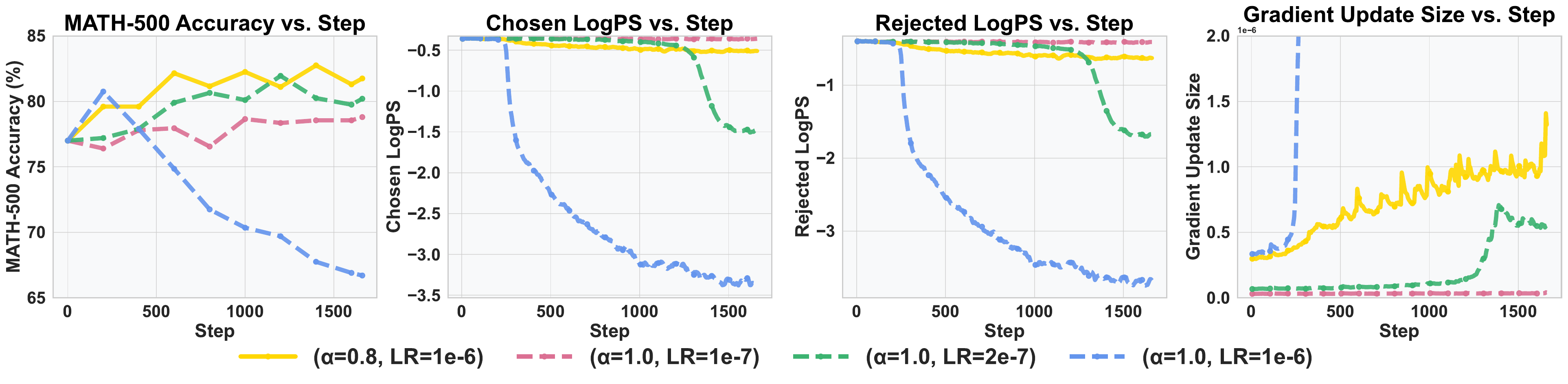}
\vspace{-15pt}
    \caption{Comparison of Symmetric REDI ($\alpha=1.0$) and Asymmetric REDI ($\alpha=0.8$).}
    \label{fig:combined_redi_metrics_plot_styled}
\vspace{-7pt}
\end{figure*}

\xpar{Symmetric REDI ($\alpha=1.0$).}
Figure~\ref{fig:combined_redi_metrics_plot_styled} shows that the Symmetric REDI objective exhibits dynamics similar to DPO with low $\beta$. A high LR ($1 \times 10^{-6}$) leads to rapid learning (peaking around 80.8\% MATH-500 accuracy) but then collapses, evidenced by sharp drops in chosen and rejected LogPS, as well as accuracy. However, reducing the learning rate significantly improves training stability. The ablation table (Table~\ref{tab:pass1_3ep}) further shows that a more stable symmetric REDI run ($\alpha=1.0, \text{LR}=2 \times 10^{-7}$) achieves 81.7\% on MATH-500, comparable to the best-tuned DPO result (81.3\%). This suggests that \textit{a simpler, REINFORCE-style and regularization-free objective can indeed match DPO's performance when its LR is carefully tuned}. Nevertheless, the trade-off between performance and stability persists. For instance, the stable $\text{LR}=1 \times 10^{-7}$ run, while avoiding LogPS collapse, achieves a lower peak accuracy than the unstable $\text{LR}=2 \times 10^{-7}$ run. This trade-off is particularly evident if we focus on the first 200 steps, where the least stable run with $\text{LR}=1 \times 10^{-6}$ achieves the highest accuracy (learns the fastest) before collapsing.

\xpar{Asymmetric weighting ($\alpha < 1.0$) is key for REDI.}
Figure~\ref{fig:combined_redi_metrics_plot_styled} (yellow solid line) demonstrates that REDI with $\alpha=0.8$ and a high LR of $1 \times 10^{-6}$ achieves rapid learning, comparable to the symmetric $\alpha=1.0$ high-LR run, but crucially, it avoids the training collapse observed in the symmetric case. It reaches a high peak performance and maintains it. The chosen and rejected LogPS do not suffer from collapse, and the gradient update size remains controlled.

\begin{insightbox}{Insight: Asymmetric weighting of positive and negative gradients improves stability and leads to higher peak performance}
Asymmetric weighting in the REDI objective (specifically, $\alpha < 1.0$, with $\alpha=0.8$ proving effective in our experiments) greatly impacts the training dynamics. In our experiments, it improves training stability and allows the use of more aggressive learning rates, which leads to faster learning and higher peak performance.
\end{insightbox}
\begin{table*}[htbp]
    \centering
    \caption{\small Model Performance Comparison (pass@1, 16 samples). We chose the best checkpoint for each configuration.}
    \label{tab:pass1_3ep}
    \resizebox{\textwidth}{!}{
    \begin{tabular}{lcccccc}
        \toprule
        Model Configuration & MATH-500 & AIME24 & AMC23 & Minerva & OlympiadBench & Avg. \\
        \midrule
        \rowcolor{SFTcolor} Qwen-SFT-1.5B-3ep (Start) & 76.7 \tiny{$\pm$0.3} & 18.1 \tiny{$\pm$1.1} & 52.8 \tiny{$\pm$0.6} & 24.6 \tiny{$\pm$0.5} & 37.5 \tiny{$\pm$0.3} & 41.9 \tiny{$\pm$0.3} \\
        \rowcolor{NeutralColor} DPO \small{($\beta=0.001, \text{LR}=2e-7$)} & 81.3 \tiny{$\pm$0.2} & 24.6 \tiny{$\pm$1.7} & 58.5 \tiny{$\pm$0.6} & 28.7 \tiny{$\pm$0.4} & 43.1 \tiny{$\pm$0.3} & 47.2 \tiny{$\pm$0.4} \\
        \rowcolor{NeutralColor} SimPO \small{($\beta=2, \gamma=1, \text{LR}=5e-7$)} & 81.1 \tiny{$\pm$0.3} & 24.8 \tiny{$\pm$1.9} & 58.8 \tiny{$\pm$0.6} & 29.1 \tiny{$\pm$0.2} & 42.2 \tiny{$\pm$0.3} & 47.2 \tiny{$\pm$0.4} \\
        \rowcolor{NeutralColor} Symmetric REDI \small{($\alpha=1.0, \text{LR}=2e-7$)} & \textbf{81.7 \tiny{$\pm$0.2}} & 25.8 \tiny{$\pm$1.3} & \textbf{59.5 \tiny{$\pm$1.1}} & 29.3 \tiny{$\pm$0.4} & 42.1 \tiny{$\pm$0.2} & 47.7 \tiny{$\pm$0.4} \\ 
        \rowcolor{REDIcolor} \textbf{REDI} \small{($\alpha=0.8, \text{LR}=1e-6$)} & \textbf{81.7 \tiny{$\pm$0.2}} & \textbf{27.3 \tiny{$\pm$1.4}} & 58.8 \tiny{$\pm$0.8} & \textbf{30.4 \tiny{$\pm$0.5}} & \textbf{43.4 \tiny{$\pm$0.3}} & \textbf{48.3 \tiny{$\pm$0.3}} \\
        \bottomrule
    \end{tabular}
    }
\end{table*}
\subsection{Tuning the Asymmetric Weighting Factor $\alpha$ in REDI}
We studied $\alpha \in \{0.2, 0.5, 0.8\}$ and found that $\alpha = 0.8$ provided the best balance for achieving strong test-time performance while maintaining stability. Lowering $\alpha$ further (e.g., to 0.5 or 0.2) lower the impact of negative gradients and tended to degrade peak performance. This is intuitive, as lower $\alpha$ values make the objective more similar to SFT on positive examples only, which we have shown to plateau earlier. We advocate setting $\alpha$ to a value like $0.8$, which is close to $1.0$, to benefit from enhanced stability without a significant sacrifice in peak performance. Refer to Appendix~\ref{subsec:appendix_alpha_tuning} for detailed ablation on $\alpha$.

\subsection{Summary of Ablation and Final Model Performance}
\xpar{Comparative analysis of REDI against established objectives.} 
Table~\ref{tab:pass1_3ep} summarizes the optimal outcomes from our ablation studies across key reasoning benchmarks (pass@1 over 16 samples), with all configurations initialized from the \texttt{Qwen-SFT-1.5B-3ep} model. Our REDI objective ($\alpha=0.8$, $\text{LR}=1 \times 10^{-6}$) consistently surpasses the SFT baseline and optimized DPO, SimPO, and symmetric REDI configurations across all metrics, achieving a benchmark average of 48.3\%.

\paragraph{Advancing data efficiency.} 
When REDI stage 2 training is applied to the stronger \texttt{Qwen-SFT-1.5B-5ep} baseline, our final \texttt{Qwen-REDI-1.5B} model attains strong results as shown in Table~\ref{tab:pass1_5ep}. Remarkably, \texttt{Qwen-REDI-1.5B}, post-trained on merely 131k openly available traces, outperforms \texttt{DeepSeek-R1-Distill-Qwen-1.5B} (trained on 800k proprietary samples). This underscores the \textbf{exceptional data efficiency} of our Reinforcement Distillation framework, achieved through systematic utilization of previously discarded negative traces.

\subsection{REDI is Robust and Generalizes Well}
\label{sec:generalizability}

To assess REDI's broader applicability, we applied it to \texttt{Llama-3.2-3B} and \texttt{Qwen2.5-Math-7B}, observing consistent gains over SFT baselines, demonstrating that learning from distilled negative traces is helpful for wider range of models. We also evaluated \texttt{Qwen-REDI-1.5B} on out-of-domain tasks (GPQA for scientific reasoning, HumanEval for code generation), finding substantial improvements despite training only on mathematical data. These results suggest REDI cultivates generalizable reasoning capabilities. Details are in Appendix~\ref{subsec:appendix_generalizability}.

\subsection{REDI Preserves Potential for Online RL}
\label{sec:redi_online_rl_potential}

We examined whether REDI's performance gains come at the cost of solution diversity by analyzing pass@16 scores. As shown in Appendix~\ref{subsec:appendix_online_rl_potential} (Tables~\ref{tab:pass16_3ep},~\ref{tab:pass16_5ep}), REDI maintains or improves pass@16 while enhancing pass@1, suggesting it broadens rather than narrows the model's capabilities, keeping models well-suited for subsequent online RL.

\section{Related Work}
\label{sec:related_work}

\paragraph{Eliciting reasoning in LLMs.}
Large reasoning models achieve strong problem solving when trained with online reinforcement learning (RL) and verifiable rewards, as in DeepSeek-R1 and follow-ups \citep{deepseekai2025deepseekr1incentivizingreasoningcapability,tinyzero,zeng2025simplerlzooinvestigatingtamingzero,yue2025limit-of-rlvr}. A commonly noted advantage of RL is its capacity to incorporate corrective signals, including penalties for errors. By contrast, standard distillation from teacher CoT traces typically relies on rejection sampling and retains only correct solutions \citep{openthoughts,openr1,bespoke_stratos,ye2025limoreasoning,muennighoff2025s1simpletesttimescaling,wen2025lightr1curriculumsftdpo}. Our approach narrows this gap by leveraging both positive and negative signals in an offline setting, improving the data efficiency of distillation.

\paragraph{Bridging Online RL with Distillation.}
A recent thread explores using teacher/expert guidance during RL to go beyond purely on-policy rollouts. LUFFY mixes off-policy demonstrations with on-policy GRPO to improve reasoning and generalization \citep{yan2025learning}. TAPO injects high-level ``thought patterns" as external guidance to augment exploration \citep{wu2025tapo}. CHORD dynamically harmonizes SFT-style expert supervision with on-policy RL via global and token-wise weighting \citep{zhang2025chord}. These works are concurrent with ours and target online RL with teacher guidance. In contrast, \textit{REDI} is fully offline: it consumes only pre-collected teacher traces and attains strong data efficiency by systematically utilizing off-policy negatives traces, without any on-policy rollouts.

\paragraph{Learning dynamics of LLM post-training.}
Instabilities in preference/RL post-training (e.g., DPO-style objectives) have been traced to negative gradient side-effects, especially when off-policy, motivating auxiliary stabilizers \citep{yan20253dpropertiesidentifyingchallengesdpo,razin2025unintentional,ren2025learning,pang2024iterative,zhangonline}. Concurrent to our research, Zhu et al. report that \emph{negative gradients are uniquely beneficial} for improving reasoning in \emph{online} RLVR, yet they also note that using negative gradients alone is highly unstable—making the \emph{relative weighting between positive and negative gradients} crucial \citep{zhu2025negative}. We study a fully \emph{offline} distillation regime, which is especially sensitive to instability, and find that moderately \emph{down weighting} negative gradients yields stable training while preserving the performance boost from utilizing learning signals from negative data.

\section{Conclusion}
\label{sec:conclusion}
We present a data-efficient distillation recipe that boosts reasoning by leveraging discarded negative traces. Our method uses a simple REINFORCE-style objective, REDI, to learn from both positive and negative signals. The effectiveness is clear: our \texttt{Qwen-REDI-1.5B} model, trained on only 131k open traces, matches a model trained on 800k proprietary traces, proving the significant value of negative distilled data.
\newpage
\section{Limitations}
Our study's scope is primarily on improving data efficiency in distillation. While our method demonstrates competitive performance, it was not explicitly designed to establish new state-of-the-art benchmarks. Furthermore, our findings are based on traces distilled from a single teacher model. Investigating whether REDI generalizes to reasoning traces from other advanced models, which may exhibit different error patterns, remains a valuable direction for future research.
\bibliography{custom}

\newpage
\appendix
\section{Convergence Guarantee}\label{sec:appendix_convergence}
\begin{definition}[REDI Loss]\label{def:loss}
    Let $\pi_\theta$ be the target policy model, which is parametrized by $\theta \in \R^d$. Let $y_w$ be the preferred data and $y_l$ be the not preferred data. Let $N$ be the number of data pairs, i.e., the number of $(y_w, y_l)$. Let $\alpha$ be a preset hyperparameter. The loss function is given as:
    \begin{equation*}
        \eqfit{\loss(\theta) := \E_{(x, y_w, y_l)\sim \mathcal{D}_{\text{Pref}}} \left[ - \tfrac{\log \pi_\theta(y_{w}\mid x)}{|y_{w}|} + \alpha \, \tfrac{\log \pi_\theta(y_{l}\mid x)}{|y_{l}|} \right]}
    \end{equation*}
\end{definition}

\begin{assumption}\label{asp:lipschitz}
    Let the loss function $\loss(\theta)$ be defined in Definition~\ref{def:loss}. We assume that $\nabla \loss(\theta)$ is $L$-Lipschitz, i.e., for $\theta,\wh{\theta} \in \R^d$, we have
    \begin{align*}
        \|\nabla \loss(\theta) - \nabla \loss(\wh{\theta}) \| \leq L \cdot \| \theta-\wh{\theta}  \|.
    \end{align*}
\end{assumption}

\begin{definition}[Update Rule with Linear Scheme]\label{def:update}
    Let the loss function $\loss(\theta)$ be defined in Definition~\ref{def:loss}. At training step $k$, for the parameter $\theta \in \R^d$, we have:
    \begin{align*}
        \theta^{k} = \theta^{k-1} - \eta_{k-1} \cdot \nabla \loss (\theta^{k-1}).
    \end{align*}
    Our Training scheme involve a linear warm-up stage and a linear decay stage. The learning rate $\eta$ is given by
    \begin{align*}
        \eta_k = \begin{cases}
            \un{\eta} + \frac{k}{\wt{k}} (\ov{\eta} -\un{\eta}), &\mathrm{if}~k \leq \wt{k};  \\
            \ov{\eta} - \frac{k - \wt{k}}{K - \wt{k}} (\ov{\eta} -\un{\eta}), &\mathrm{if}~k > \wt{k},
        \end{cases}
    \end{align*}
    where $\wt{k} \in [K]$ is a preset hyperparameter denoting the number of warm-up steps, $\un{\eta}$ and $\ov{\eta}$ are the minimum value and maximum value of learning rate $\eta$ respectively, i.e., $\eta \in [\un{\eta},\ov{\eta}]$. Specifically, we set $\ov{\eta}<1/L$, where $L$ is the Lipschitz constant in Assumption~\ref{asp:lipschitz}.
\end{definition}

\begin{theorem}[Convergence Guarantee]\label{thm:conv}
    Let the loss function $\loss (\theta)$ be defined in Definition~\ref{def:loss}. For any small $\epsilon >0$, the update iterations satisfy:
    \begin{align*}
        \min_{k \in [K]} \E[\| \nabla \loss(\theta^k)\|^2 ] \leq \epsilon.
    \end{align*}
\end{theorem}

\begin{proof}
    At time step $k-1$, we analyze the expected loss and perform a Taylor expansion of $\loss(\theta)$:
    \begin{equation*}
        \eqfit{\begin{aligned}
            \E [\loss (\theta^{k})] 
            &\leq \E\big[\loss(\theta^{k-1}) + (\theta^{k} - \theta^{k-1})^\top \nabla \loss(\theta^{k-1}) + 0.5 L \, \|\theta^{k} - \theta^{k-1} \|^2\big]\\
            &\leq \E\big[\loss(\theta^{k-1}) - \eta_{k-1} \, \nabla \loss(\theta^{k-1}) ^\top \nabla \loss(\theta^{k-1}) + 0.5 L \, \| \eta_{k-1} \, \nabla \loss(\theta^{k-1})\|^2\big]\\
            &= \E[\loss(\theta^{k-1})] - \eta_{k-1} \, \E[\| \nabla \loss (\theta^{k-1})\|^2] + 0.5L\eta_{k-1}^2 \, \E[\| \nabla \loss (\theta^{k-1})\|^2],
        \end{aligned}}
    \end{equation*}
    where the first step follows from Assumption~\ref{asp:lipschitz}, the second step follows from Definition~\ref{def:update}, and the third step follows from basic algebra.

    Thus, we can show that
    \begin{equation}\label{eq:conv1}
        \eqfit{\E [ \| \nabla \loss (\theta^{k-1})\|^2] 
        \leq \frac{1}{\eta (1-0.5 L \eta)}\,\E [\loss (\theta^{k-1}) - \loss (\theta^{k})]}
    \end{equation}
    which follows from Definition~\ref{def:update} and basic algebra.
    
    Further, for the minimal value of $\E[\| \nabla \loss(\theta^{k-1}) \|^2]$, we have
    \begin{equation*}
        \eqfit{\begin{aligned}
            \min_{k \in [K]} \E[\| \nabla \loss(\theta^{k-1}) \|^2] 
            &\leq \frac{1}{K} \sum_{k = 1}^{K} \E[\| \nabla \loss(\theta^{k-1}) \|^2] \\
            &\leq \frac{1}{K} \sum_{k = 1}^{K} \Big(\tfrac{1}{\eta_{k-1} (1-0.5 L \eta_{k-1})}\,\E [\loss (\theta^{k-1}) - \loss (\theta^{k})]\Big) \\
            &\leq \tfrac{1}{K \un{\eta} (1-0.5 L \ov{\eta})}\,\big(\loss (\theta^{0}) - \loss (\theta^{K})\big)\\
            &\leq \tfrac{1}{K \un{\eta} (1-0.5 L \ov{\eta})}\,\big(\loss (\theta^{0}) - \loss (\theta^{*})\big)
        \end{aligned}}
    \end{equation*}
    where the first step follows from the minimum is always smaller than the average, the second step follows from Eq.~\eqref{eq:conv1}, the third step follows from Definition~\ref{def:update} and basic algebra, the fourth step follows from $\loss (\theta^*) \leq \loss (\theta^K)$.
    
    Plugging in 
    \begin{align*}
        K = \frac{\loss (\theta^{0}) - \loss (\theta^{*})}{\un{\eta} (1-0.5 L \ov{\eta}) \epsilon},
    \end{align*}
    we finish the proof.
\end{proof}

\section{Relationship between SimPO and Our Loss Function}
\label{sec:appendix_simpo_relationship_derivation}
First, we restate SimPO loss Eq.~\eqref{eq:simpo_loss} as follows:
\begin{equation*}
    \eqfit{\loss_{\text{SimPO}}(\theta) = - \E_{(x, y_w, y_l) \,\sim\, \mathcal{D}_{\text{Pref}}} \left[ \log \sigma \left( \beta \left( \tfrac{\log \pi_\theta(y_w\mid x)}{|y_w|} - \tfrac{\log \pi_\theta(y_l\mid x)}{|y_l|} \right)  - \gamma \right) \right]}
\end{equation*}

Then, we restate the gradient of SimPO, which is implicit on page 22 in~\cite{meng2024simpo}.
\begin{equation*}
    \eqfit{\begin{aligned}
        \nabla_\theta \loss_{\text{SimPO}}(\theta) 
        =\; & - \beta \E_{(x,y_w,t_l) \,\sim\, \mathcal{D}_\mathrm{Pref}}\Big[\sigma\!\Big(\tfrac{\beta}{|y_l|}\log(y_l\mid x) - \tfrac{\beta}{|y_w|}\log(y_w\mid x) + \gamma\Big) \cdot\\
        & \Big(\tfrac{1}{|y_w|} \, \nabla_\theta \log(y_w\mid x) - \tfrac{1}{|y_l|} \, \nabla_\theta \log(y_l\mid x)\Big)\Big].
    \end{aligned}}
\end{equation*}
Define $R_\theta := \frac{1}{|y_w|} \log(y_w~|~x) - \frac{1}{|y_l|} \log(y_l~|~x)$, we have
\begin{equation*}
    \eqfit{\begin{aligned}
        \nabla_\theta \loss_{\text{SimPO}}(\theta) 
        =\; & - \beta \E_{(x,y_w,t_l) \,\sim\, \mathcal{D}_\mathrm{Pref}}\big[\sigma(- \beta \cdot R_\theta + \gamma) \cdot \nabla_\theta R_\theta\big]\\
        =\; & - \E_{(x,y_w,t_l) \,\sim\, \mathcal{D}_\mathrm{Pref}}\big[\sigma(- \beta \cdot R_\theta + \gamma) \cdot \beta \, \nabla_\theta  R_\theta\big].
    \end{aligned}}
\end{equation*}
Also, we define
\begin{equation*}
    \eqfit{\loss_{\text{symm}}(\theta) = \E_{(x, y_w, y_l)\sim \mathcal{D}_{\text{Pref}}} \left( - \tfrac{\log \pi_\theta(y_{w}\mid x)}{|y_{w}|} + \tfrac{\log \pi_\theta(y_{l}\mid x)}{|y_{l}|} \right)}
\end{equation*}

\begin{fact}\label{fac:sigmoid_L}
    the sigmoid function is Lipschitz continuous, i.e., $|\sigma(x) - \sigma(x')| \leq 0.25|x - x'|$.
\end{fact}

\begin{lemma}
    Let $R_\theta$ be bounded by constant $c_0$. 
    Let the hyperparameter $\beta>0$ be an arbitrary small number.
    Let $\gamma$ be a constant.
    Let $\epsilon = \frac{\beta \cdot c_0}{4}$ which can be arbitary small, we have
    \begin{align*}
        | \sigma (-\beta R_\theta + \gamma) - \sigma (\gamma) | < \epsilon.
    \end{align*}
\end{lemma}
\begin{proof}
    We can show
    \begin{align*}
        | \sigma (-\beta R + \gamma) - \sigma (\gamma) | 
        \leq ~ & \frac{\beta \cdot R_\theta}{4}\\
        \leq ~ & \frac{\beta \cdot c_0}{4},
    \end{align*}
    where the first step follows from Fact~\ref{fac:sigmoid_L}, the second step follows from $R_\theta$ is bounded by constant $c_0$.
\end{proof}

Given typical learning rate $\eta$,  adjusted learning rate $\eta'$. 
We claim the one step update over parameter $\theta$ with loss function $\loss_{\text{SimPO}}(\theta)$ is approximately equal to $\loss_{\text{symm}}(\theta)$.

\begin{proposition}
    Let the hyperparameter $\beta>0$ be an arbitrary small number. Let $\eta_\mathrm{SimPO}$ be set as an inverse multiple of $\beta$, i.e., $\eta_\mathrm{SimPO} = c_1/\beta$. Assume $\nabla_\theta R_\theta$ is bounded. Given the initial parameter $\theta_t$. We can choose learning rate $\eta = c_1 \cdot \sigma(\gamma)$, such that for the one step gradient decent update over 
        $
        \Delta^\mathrm{SimPO}_\theta := \theta_{t+1} - \theta_{t} = -\eta_\mathrm{SimPO} \nabla_\theta \loss_{\text{SimPO}}(\theta)
        $
     is approximately equal to the  gradient decent update over 
    $
         \Delta^\mathrm{symm}_\theta := \theta_{t+1} - \theta_{t} = -\eta \nabla_\theta \loss_{\text{symm}}(\theta),
    $
    i.e., $|\eta_\mathrm{SimPO} \nabla_\theta \loss_{\text{SimPO}}(\theta) - \eta \nabla_\theta \loss_{\text{symm}}(\theta)|$ can be arbitrary small.
\end{proposition}
\begin{proof}
We have
\begin{equation*}
    \eqfit{\begin{aligned}
        &|\eta_\mathrm{SimPO} \, \nabla_\theta \, \loss_{\text{SimPO}}(\theta) - \eta \, \nabla_\theta \, \loss_{\text{symm}}(\theta)| \\
        &= |\eta_\mathrm{SimPO} \, \sigma(- \beta \cdot R_\theta + \gamma) \cdot \beta \, \nabla_\theta  R_\theta - \eta \, \nabla_\theta R_\theta| \\
        &= |c_1 \, \sigma(- \beta \cdot R_\theta + \gamma) \, \nabla_\theta  R_\theta - \eta \, \nabla_\theta R_\theta|\\
        &= |c_1 \, \sigma(- \beta \cdot R_\theta + \gamma) \, \nabla_\theta  R_\theta - c_1 \, \sigma(\gamma) \, \nabla_\theta R_\theta|\\
        &= c_1 \, \|\nabla_\theta R_\theta\| \, | \sigma (-\beta R + \gamma) - \sigma (\gamma) | \\
        &\leq c_0 \, c_1 \, \|\nabla_\theta R_\theta\| \, \beta / 4,
    \end{aligned}}
\end{equation*}
    which can be arbitrarily small.
\end{proof}

\section{Detailed Experimental Setup}
\label{sec:appendix_experimental_setup}

This section provides a comprehensive overview of the experimental setup, including details on the base model, dataset curation, training configurations for both SFT and preference optimization stages, evaluation protocols, and computational resources. For more detailed implementation, readers may refer to the provided codebase.

\subsection{Base Model and Initial SFT Checkpoints}
All experiments commenced with the \texttt{Qwen2.5-Math-1.5B} model as the base LLM, chosen for its strong foundational capabilities in mathematical reasoning. Two SFT checkpoints were prepared from this base model to serve different purposes:
\begin{itemize}[topsep=1pt,itemsep=1pt,leftmargin=15pt]
    \item \texttt{Qwen-SFT-1.5B-3ep}: This model was fine-tuned on the $\mathcal{D}_{\text{SFT}}$ dataset for 3 epochs. It served as the starting point for the ablation studies involving DPO, SimPO, and REDI, as detailed in Section~\ref{sec:results} and Table~\ref{tab:pass1_3ep}.
    \item \texttt{Qwen-SFT-1.5B-5ep}: Fine-tuned for 5 epochs on $\mathcal{D}_{\text{SFT}}$, this model demonstrated improved SFT performance and was used as the SFT starting point for our final, best-performing \texttt{Qwen-REDI-1.5B} model presented in Table~\ref{tab:pass1_5ep}.
\end{itemize}

\subsection{Datasets}
\label{subsec:appendix_data}
As described in Section~\ref{subsec:base_model_datasets}, the data was derived from the OpenR1-Math-Raw corpus \citep{openr1}, excluding the \texttt{cn\_k12} subset due to its lower relative difficulty. A response was considered correct if both the Llama judge (an LLM-based verifier) and Math-Verify \citep{Kydlicek_Math-Verify_Math_Verification} (a rule-based verifier) labeled it as ``True"; otherwise, it was considered incorrect.

\begin{itemize}[topsep=1pt,itemsep=1pt,leftmargin=15pt]
    \item \textbf{$\mathcal{D}_{\text{SFT}}$ (Positive Traces Dataset):} Contained 77,629 problem-solution pairs $(x, y_w)$ where $y_w$ is a correct reasoning trace. This dataset was used for Stage 1 SFT.
    \item \textbf{$\mathcal{D}_{\text{Pref}}$ (Preference Pairs Dataset):} Consisted of 53,175 triplets $(x, y_w, y_l)$. This dataset was derived by selecting data from $\mathcal{D}_{\text{SFT}}$ for which an incorrect response $y_l$ (deemed incorrect by either Math-Verify or the Llama verifier) was also available for the same problem $x$. Each triplet comprises a problem $x$, a preferred correct trace $y_w$, and a rejected incorrect trace $y_l$. We further filtered out instances where queries exceeded 800 tokens, or either chosen ($y_w$) or rejected ($y_l$) responses exceeded 19,000 tokens. This dataset was used for Stage 2 preference optimization.
\end{itemize}

\subsection{Stage 1: Supervised Fine-Tuning (SFT)}
\begin{itemize}[topsep=1pt,itemsep=1pt,leftmargin=15pt]
    \item \textbf{Objective:} Maximize log-likelihood of positive traces (Eq.~\eqref{eq:sft_loss}).
    \item \textbf{Optimizer:} AdamW \citep{loshchilov2019decoupledweightdecayregularization} with $\beta_1=0.9$, $\beta_2=0.999$, $\epsilon=10^{-8}$, and a weight decay of $0.0001$.
    \item \textbf{Learning Rate Schedule:} Peak learning rate of $5 \times 10^{-5}$, with a linear warmup for the first 10\% of total training steps, followed by a linear decay to zero.
    \item \textbf{Batch Size:} 128.
    \item \textbf{Epochs:} 3 epochs for \texttt{Qwen-SFT-1.5B-3ep} and 5 epochs for \texttt{Qwen-SFT-1.5B-5ep}.
    \item \textbf{Max Sequence Length:} 32,768 tokens.
\end{itemize}

\subsection{Stage 2: Preference Optimization}
\label{subsec:appendix_hyperparam}
All preference optimization methods (DPO, SimPO, REDI) were initialized from an SFT checkpoint (\texttt{Qwen-SFT-1.5B-3ep} for ablations, \texttt{Qwen-SFT-1.5B-5ep} for the final model). Training was conducted on the $\mathcal{D}_{\text{Pref}}$ dataset.
\begin{itemize}[topsep=1pt,itemsep=1pt,leftmargin=15pt]
    \item \textbf{Optimizer:} AdamW with the same parameters as in Stage 1 ($\beta_1=0.9, \beta_2=0.999, \epsilon=10^{-8}$, weight decay $0.0001$).
    \item \textbf{Learning Rate Schedule:} Linear warmup for the first 10\% of total training steps, followed by linear decay to zero. Peak learning rates were method-specific and tuned as described below.
    \item \textbf{Batch Size:} 32.
    \item \textbf{Epochs:} 1 epoch over $\mathcal{D}_{\text{Pref}}$.
    \item \textbf{Max Sequence Length:} 800 tokens for queries (prompts $x$) and 19,000 tokens for responses ($y_w, y_l$).
\end{itemize}

\xpar{Hyperparameter Configurations and Tuning:}
We meticulously tuned hyperparameters for each preference optimization method. The reference model for DPO was the \texttt{Qwen-SFT-1.5B-3ep} checkpoint.

\begin{itemize}[topsep=1pt,itemsep=1pt,leftmargin=15pt]
    \item \textbf{DPO \citep{rafailov2024directpreferenceoptimizationlanguage}:}
        \begin{itemize}[topsep=0pt,itemsep=0pt,leftmargin=10pt]
            \item \textbf{$\beta$ values explored:} $\{0.001, 0.01, 0.1\}$.
            \item \textbf{Learning Rate (LR) exploration:} Specific LRs were tested for each $\beta$ value:
                \begin{itemize}[topsep=0pt,itemsep=0pt,leftmargin=8pt]
                    \item For $\beta=0.1$: $\{1 \times 10^{-6}, 1 \times 10^{-7}, 1 \times 10^{-8}\}$.
                    \item For $\beta=0.01$: $\{1 \times 10^{-6}, 5 \times 10^{-7}, 2 \times 10^{-7}, 1 \times 10^{-7}\}$.
                    \item For $\beta=0.001$: $\{1 \times 10^{-6}, 5 \times 10^{-7}, 2 \times 10^{-7}, 1 \times 10^{-7}\}$.
                \end{itemize}
            \item \textbf{Best Ablation Configuration (Table~\ref{tab:pass1_3ep}):} $\beta=0.001$, $\text{LR}=2 \times 10^{-7}$.
        \end{itemize}

    \item \textbf{SimPO \citep{meng2024simpo}:}
        \begin{itemize}[topsep=0pt,itemsep=0pt,leftmargin=10pt]
            \item Two primary configurations were evaluated based on different $\beta$ values, with the margin $\gamma$ set to maintain a $\gamma/\beta$ ratio of 0.5:
                \begin{itemize}[topsep=0pt,itemsep=0pt,leftmargin=8pt]
                    \item Configuration 1: $\beta=2, \gamma=1, \text{LR}=5 \times 10^{-7}$.
                    \item Configuration 2: $\beta=10, \gamma=5, \text{LR}=3 \times 10^{-7}$.
                \end{itemize}
            \item \textbf{Best Ablation Configuration (Table~\ref{tab:pass1_3ep}):} The first configuration, $\beta=2, \gamma=1, \text{LR}=5 \times 10^{-7}$, yielded superior results in our ablation studies.
        \end{itemize}

    \item \textbf{REDI (Ours):}
        \begin{itemize}[topsep=0pt,itemsep=0pt,leftmargin=10pt]
            \item \textbf{$\alpha$ (asymmetric weight) values explored:} $\{0.2, 0.5, 0.8, 1.0\}$. (See Appendix~\ref{sec:additional_results} for detailed $\alpha$ tuning).
            \item \textbf{Learning Rate (LR) exploration:}
                \begin{itemize}[topsep=0pt,itemsep=0pt,leftmargin=8pt]
                    \item For $\alpha \in \{0.2, 0.8\}$, a learning rate of $1 \times 10^{-6}$ was primarily used.
                    \item For $\alpha=0.5$, learning rates of $1 \times 10^{-6}$ and $2 \times 10^{-6}$ were tested.
                    \item For $\alpha=1.0$ (Symmetric REDI), learning rates of $\{1 \times 10^{-6}, 2 \times 10^{-7}, 1 \times 10^{-7}\}$ were evaluated.
                \end{itemize}
            \item \textbf{Best Ablation Configuration (Table~\ref{tab:pass1_3ep}):} $\alpha=0.8, \text{LR}=1 \times 10^{-6}$.
            \item \textbf{Final Model Configuration (Table~\ref{tab:pass1_5ep}):} For the \texttt{Qwen-REDI-1.5B} model, initialized from \texttt{Qwen-SFT-1.5B-5ep}, we used $\alpha=0.8$ with $\text{LR}=1 \times 10^{-6}$.
        \end{itemize}
\end{itemize}

\subsection{Evaluation}
\label{subsec:appendix_evaluation_details}

\xpar{Decoding Parameters:}
During all evaluations, generated samples were decoded using the following parameters:
\begin{itemize}[topsep=1pt,itemsep=1pt,leftmargin=15pt]
    \item Temperature: $0.6$
    \item Top P (nucleus sampling): $p=0.95$
    \item Maximum generation length: $32,768$ tokens
\end{itemize}

\xpar{Evaluation Frameworks, Protocols, and Benchmarks:}
We utilized two distinct configurations for evaluating model performance:
\begin{itemize}[topsep=1pt,itemsep=1pt,leftmargin=15pt]
    \item \textit{Intermediate Evaluations:} These evaluations were used for hyperparameter tuning, generating performance plots (e.g., Figures \ref{fig:dpo_dynamics}, \ref{fig:dpo_lr_sweeps}, \ref{fig:simpo_dynamics}, \ref{fig:combined_redi_metrics_plot_styled}), and ablation studies. They were conducted using LightEval \citep{lighteval} on the MATH-500 benchmark. Performance was measured as pass@1, averaged over 4 generated samples per problem.
    \item \textit{Final Model Evaluations:} These evaluations, presented in main comparison tables (Table~\ref{tab:pass1_5ep},\ref{tab:pass1_3ep},\ref{tab:pass16_3ep},\ref{tab:pass16_5ep}), were performed using the DeepScaleR/rllm \citep{deepscaler2025} framework. The benchmarks included MATH-500, AIME24, AMC23, Minerva, and OlympiadBench. Performance was measured as either pass@1 (averaged over 16 generated samples) or pass@16. For these evaluations, and specifically for our models, we fixed ``<think>" as the first token generated by our model to align our practices with DeepSeek-R1 series of models.
\end{itemize}

\xpar{SEM Calculation:}
\label{subsec:appendix_sem}
The reported Standard Error of the Mean (SEM) quantifies the uncertainty of this ``pass@1 over $k$ samples" score. It is calculated as $s/\sqrt{k}$. To obtain $s$, we first compute $k$ distinct ``benchmark-wide pass@1 scores." Each of these $k$ scores ($P_j$, for $j=1 \dots k$) is determined by evaluating the model's performance across the entire benchmark using \textit{only the $j$-th generated sample} for every problem. The term $s$ is then the standard deviation of these $k$ intermediate scores ($P_1, P_2, \dots, P_k$). This method estimates the variability of the overall ``pass@1 over $k$ samples" metric by assessing performance consistency across the individual samples drawn for each problem.

\xpar{Evaluation Prompt Format:}
For prompting, we followed the Open-R1 project \cite{openr1} and used the following template:

\begin{promptbox}{Prompt Template}
\texttt{Solve the following math problem efficiently and clearly.  The last line of your response should be of the following format: `Therefore, the final answer is: \$\textbackslash boxed\{\{ANSWER\}\}\$. I hope it is correct' (without quotes) where ANSWER is just the final number or expression that solves the problem. Think step by step before answering.\\\\
\{Question\}}
\end{promptbox}

\subsection{Computational Resources}
All model training and fine-tuning experiments were conducted on a distributed training cluster equipped with NVIDIA A100 80GB SXM GPUs. Each experiment was run on a node of 8 such GPUs. We utilized standard open-source libraries for large language model training, including Hugging Face Transformers for model architecture and tokenization, and Accelerate for distributed training management. DeepSpeed (ZeRO Stage 3 for DPO; ZeRO Stage 2 for SimPO and REDI) was employed to optimize memory usage and enable efficient training. Custom scripts were developed for data processing; the computational resources required for preprocessing were negligible.
\begin{table*}[t]
    \centering
    \caption{Approximate Training Times per Run (on one 8-GPU A100 80GB node).}
    \label{tab:training_times_appendix}
    \begin{tabular}{lcc}
        \toprule
        Task / Method                         & Duration per Run & Approx. GPU Hours \\
        \midrule
        SFT (3 epochs on $\mathcal{D}_{\text{SFT}}$) & 8 hours          & 64                \\
        SFT (5 epochs on $\mathcal{D}_{\text{SFT}}$) & 13 hours         & 104               \\
        DPO (1 epoch on $\mathcal{D}_{\text{Pref}}$) & ~5 hours 10 mins & ~41.3             \\
        SimPO (1 epoch on $\mathcal{D}_{\text{Pref}}$) & ~4 hours 50 mins & ~38.7             \\
        REDI (1 epoch on $\mathcal{D}_{\text{Pref}}$) & ~4 hours 10 mins & ~33.3             \\
        \bottomrule
    \end{tabular}
\end{table*}
\xpar{Training Times:}
The approximate training times per run on an 8-GPU (A100 80GB) node are summarized in Table~\ref{tab:training_times_appendix}.

The total training compute required for our final \texttt{Qwen-REDI-1.5B} model (SFT 5 epochs + REDI 1 epoch) is approximately 17 hours on an 8-GPU node (or around 136 A100 80GB GPU hours).

\xpar{Evaluation Times:}
The approximate evaluation times on an 8-GPU (A100 80GB) node are:
\begin{itemize}[topsep=1pt,itemsep=1pt,leftmargin=15pt]
    \item MATH-500 (pass@1 over 4 samples, LightEval): ~40 minutes.
    \item 5 Benchmarks (pass@1 over 16 samples, DeepScaleR/rllm): ~20 hours.
\end{itemize}

\xpar{Total Compute for Reproducibility:}
The total compute needed to reproduce all results presented in this paper (including all SFT runs, hyperparameter sweeps for DPO, SimPO, and REDI, final evaluations and with buffer for debugging) is estimated to be around 350 hours on an 8-GPU node, which translates to approximately 2,800 A100 80GB GPU hours.

\section{Additional Results}
\label{sec:additional_results}
\subsection{Ablation Study on REDI Hyperparameter $\alpha$}
\label{subsec:appendix_alpha_tuning}

\begin{figure*}[t]
    \centering
    \includegraphics[width=1.0\textwidth]{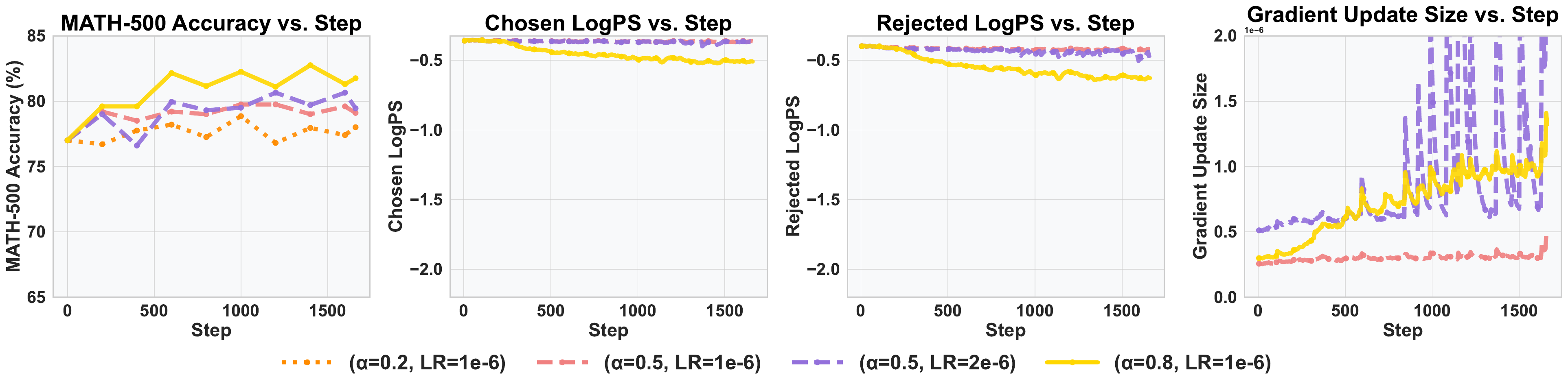}
    \caption{REDI training dynamics with varying $\alpha$ values and learning rates. All runs start from \texttt{Qwen-SFT-1.5B-3ep}. Metrics shown are MATH-500 Accuracy, Chosen LogPS, Rejected LogPS, and Gradient Update Size, all plotted against training steps.}
    \label{fig:alpha_tuning_plot_appendix}
\end{figure*}

This section details the ablation study conducted to determine an effective value for the asymmetric weighting hyperparameter $\alpha$ in the REDI objective (Equation~\ref{eq:redi_loss_expected}). The goal was to find an $\alpha$ that optimally leverages negative traces for performance improvement while maintaining training stability. We explored $\alpha \in \{0.2, 0.5, 0.8\}$ (with $\alpha=1.0$ representing the symmetric case discussed in the main paper), using the \texttt{Qwen-SFT-1.5B-3ep} model as the starting checkpoint. Figure~\ref{fig:alpha_tuning_plot_appendix} illustrates the training dynamics for key configurations.

\xpar{Analysis of Different $\alpha$ Configurations:}
\begin{itemize}[topsep=1pt,itemsep=2pt,leftmargin=20pt]
    \item \textbf{Low $\alpha$ ($\alpha=0.2, \text{LR}=1 \times 10^{-6}$, orange dotted line):} With this configuration, the model's performance largely fluctuated around the initial SFT level. This is anticipated, as a lower $\alpha$ value makes the REDI objective more closely resemble the SFT loss, for which performance had already neared a plateau.

    \item \textbf{Moderate $\alpha$ ($\alpha=0.5, \text{LR}=1 \times 10^{-6}$, pink dashed line):} Increasing $\alpha$ to $0.5$ while maintaining $\text{LR}=1 \times 10^{-6}$ yielded improved peak accuracy (approximately 79.8\%) compared to $\alpha=0.2$. This underscores the benefit of incorporating negative samples, even moderately, over relying solely on SFT.

    \item \textbf{Moderate $\alpha$ with Higher LR ($\alpha=0.5, \text{LR}=2 \times 10^{-6}$, purple dashed line):} Testing $\alpha=0.5$ with a more aggressive learning rate ($\text{LR}=2 \times 10^{-6}$) showed stable LogPS values, though accompanied by intermittent spikes in the gradient update size. While its peak performance slightly surpassed the lower LR variant with $\alpha=0.5$, it remained inferior to the $\alpha=0.8$ run.

    \item \textbf{Higher $\alpha$ ($\alpha=0.8, \text{LR}=1 \times 10^{-6}$, yellow solid line):} This configuration achieved the highest peak MATH-500 accuracy. We note that both the Chosen LogPS and Rejected LogPS steadily decreased throughout training. This concurrent decrease, in the absence of a sudden collapse and while performance is improving, appears to be benign. It is distinct from a catastrophic collapse where both LogPS would plummet sharply alongside performance.
\end{itemize}

\xpar{Importance of Update Direction over Raw Magnitude:}
The comparison between the $(\alpha=0.5, \text{LR}=2 \times 10^{-6})$ and $(\alpha=0.8, \text{LR}=1 \times 10^{-6})$ configurations is particularly insightful. The former features larger average gradient update sizes, implying a stronger raw \textit{magnitude} of adjustment driven by the negative log-likelihood term (which scales with $\alpha \times \text{LR}$). However, this did not translate to superior peak performance relative to the $(\alpha=0.8, \text{LR}=1 \times 10^{-6})$ run.

This observation supports the view that the \textit{direction} of the gradient, as modulated by $\alpha$, is more critical than its sheer magnitude derived from negative samples. A higher $\alpha$ (such as 0.8) appears to provide a more qualitatively beneficial gradient signal, guiding the model more effectively. Simply increasing the learning rate for a lower $\alpha$ (e.g., $\alpha=0.5$) to match or exceed the raw gradient magnitude of a higher $\alpha$ configuration does not necessarily yield better performance and may even compromise stability. The role of $\alpha$ thus extends beyond scaling the penalty; it is crucial for appropriately balancing the influence of negative examples to effectively shape the learning landscape.

\xpar{Suggestions on $\alpha$ Tuning:}
Based on these ablation studies, $\alpha=0.8$ combined with a learning rate of $1 \times 10^{-6}$ demonstrated the most favorable trade-off for the \texttt{Qwen-SFT-1.5B-3ep} checkpoint, achieving the highest peak performance while maintaining robust training stability. Consequently, for applying the REDI framework to other domains or datasets, we recommend initially fixing $\alpha = 0.8$ and primarily focusing on tuning the learning rate.

\subsection{Generalizability and Robustness of REDI}
\label{subsec:appendix_generalizability}

To validate that the benefits of the REDI framework extend beyond a single model family and task domain, we conducted a series of additional experiments testing its generalizability across different model architectures, scales, and out-of-domain reasoning tasks. REDI demonstrates storng results. For further details, refer to appendix

\paragraph{Generalization Across Architectures and Scales.}
We applied the REDI framework to two additional models: \texttt{Llama-3.2-3B}, representing a different model architecture, and \texttt{Qwen2.5-Math-7B}, a larger model within the same family. Following the same two-stage training protocol, we first established a strong Rejection Sampling SFT baseline for each model and then applied one epoch of REDI training. As shown in Table~\ref{tab:generalization_arch_scale}, REDI consistently delivers significant performance gains over the SFT baseline in both cases. On \texttt{Llama-3.2-3B}, REDI achieves a relative improvement of 17.7\% on the benchmark average, and on \texttt{Qwen2.5-Math-7B}, it improves the average score by 4.2\%. These results strongly suggest that REDI's effectiveness in leveraging negative signals is a general principle applicable across modern decoder-only transformer architectures and is not limited to a specific model size.

\begin{table*}[htbp]
    \centering
    \caption{\small \textbf{Generalization Across Architectures and Scales (pass@1 over 16 samples).} REDI consistently improves performance over the strong SFT baseline on both the \texttt{Llama-3.2-3B} and the larger \texttt{Qwen2.5-Math-7B} models.}
    \label{tab:generalization_arch_scale}
    \resizebox{\textwidth}{!}{
    \begin{tabular}{llcccccc}
        \toprule
        \textbf{Model} & \textbf{Method} & \textbf{MATH-500} & \textbf{AIME24} & \textbf{AMC23} & \textbf{Minerva} & \textbf{OlympiadBench} & \textbf{Avg.} \\
        \midrule
        \rowcolor{NeutralColor!20}
        \texttt{Llama-3.2-3B} & SFT & 54.3 & 8.5 & 33.9 & 14.3 & 28.8 & 28.0 \\
        \rowcolor{REDIcolor!25}
        & \textbf{REDI} & \textbf{61.2} & \textbf{13.3} & \textbf{41.9} & \textbf{15.3} & \textbf{32.9} & \textbf{32.9} \\
        \midrule
        \rowcolor{NeutralColor!20}
        \texttt{Qwen2.5-Math-7B} & SFT & 90.3 & 51.0 & 75.9 & 39.8 & 55.3 & 62.5 \\
        \rowcolor{REDIcolor!25}
        & \textbf{REDI} & \textbf{91.2} & \textbf{56.5} & \textbf{79.6} & \textbf{40.5} & \textbf{57.6} & \textbf{65.1} \\
        \bottomrule
    \end{tabular}
    }
\end{table*}

\paragraph{Generalization to Out-of-Domain Tasks.}
A critical question is whether the improved reasoning abilities fostered by REDI on mathematical data can transfer to other complex domains. To investigate this, we evaluated our final Qwen-REDI-1.5B model against its SFT baseline (Qwen-SFT-1.5B-3ep) on the out-of-domain benchmarks GPQA \citep{gpqa} for scientific reasoning and HumanEval \citep{chen2021codex} for code generation. The evaluation on GPQA was done with LightEval \cite{lighteval}, and the evaluation on HumanEval was done with lm-evaluation-harness \cite{eval-harness}. We note that our model lacks instruction-following capabilities for HumanEval, so we used the default prompt setting in lm-evaluation-harness, which is a continuation prompt.

The results, presented in Table~\ref{tab:ood_generalization}, demonstrate strong and positive transfer. REDI achieves a 26.8\% relative improvement on GPQA and a 181.2\% improvement in pass@1 on HumanEval. This suggests that by learning from both positive and negative examples, REDI cultivates a more robust and generalizable reasoning policy that is less prone to overfitting on the distribution of the distilled math data, enabling it to better follow instructions and solve problems in unrelated domains.

\begin{table*}[htbp]
    \centering
    \caption{\small \textbf{Out-of-Domain Generalization to Science QA and Coding.} The reasoning capabilities gained from REDI on math data transfer effectively to other complex domains, with particularly strong improvements on HumanEval.}
    \label{tab:ood_generalization}
    \resizebox{\textwidth}{!}{
    \begin{tabular}{lcccc}
        \toprule
        \textbf{Model} & \textbf{GPQA} (pass@1) & \textbf{OlympiadBench} (pass@1) & \textbf{HumanEval} (pass@1) & \textbf{HumanEval} (pass@5) \\
        \midrule
        \rowcolor{SFTcolor}
        Qwen-SFT-1.5B-3ep & 28.3\% & 37.5\% & 7.1\% & 22.6\% \\
        \rowcolor{REDIcolor}
        \textbf{Qwen-REDI-1.5B} & \textbf{35.9\%} & \textbf{43.4\%} & \textbf{19.9\%} & \textbf{39.6\%} \\
        \midrule
        \textit{Relative Improvement} & \textit{+26.8\%} & \textit{+15.7\%} & \textit{+181.2\%} & \textit{+75.7\%} \\
        \bottomrule
    \end{tabular}
    }
\end{table*}

\subsection{REDI Improves Performance Without Harming Potential for Future Online RL}
\label{subsec:appendix_online_rl_potential}

A key consideration is whether REDI enhances performance (like pass@1) by simply reinforcing the model's existing high-probability solution paths, or if it genuinely broadens its problem-solving abilities. Online Reinforcement Learning (RL) often works by refining and amplifying the knowledge already present within a model \citep{shao2024deepseekmathpushinglimitsmathematical, yue2025limit-of-rlvr}. Therefore, it's important that an offline method like REDI doesn't narrow the model's underlying knowledge base.

A model's ability to find a correct answer given multiple attempts (e.g., pass@k for larger $k$, like $k=16$) can serve as an indicator of the breadth of its existing knowledge. If REDI maintains or improves these pass@k scores, it suggests that while it refines certain solution strategies, it doesn't do so at the expense of the model's diverse underlying capabilities. This would mean the model remains a strong candidate for subsequent online RL.

We investigate this by examining pass@16 performance, as presented in Tables~\ref{tab:pass16_3ep} and~\ref{tab:pass16_5ep}.

Table~\ref{tab:pass16_3ep} shows that for models initialized from \texttt{Qwen-SFT-1.5B-3ep}, REDI (with $\alpha=0.8$) not only improves pass@1 (Table~\ref{tab:pass1_3ep}) but also sustains or improves pass@16 scores across several benchmarks (e.g., AIME24, Minerva, OlympiadBench) compared to both the SFT baseline and other preference optimization methods. For instance, it achieved the best pass@16 on AIME24 and Minerva among the preference-tuned models.

Furthermore, Table~\ref{tab:pass16_5ep} indicates that our final \texttt{Qwen-REDI-1.5B} model (initialized from the stronger \texttt{Qwen-SFT-1.5B-5ep}) maintains robust pass@16 performance. It achieves the highest pass@16 on AIME24 and matches or surpasses the SFT baseline and the \texttt{DeepSeek-R1-Distill-Qwen-1.5B} model on Minerva and OlympiadBench.

The consistent maintenance or improvement in pass@16 scores suggests that REDI's offline refinement does not merely over-optimize for a narrow set of high-probability solutions from the SFT model. Rather, these pass@16 results indicate that by learning from both the teacher's successful and unsuccessful solution attempts, REDI genuinely improves the model's overall problem-solving abilities. It appears to build these skills without causing the model to ``forget" or narrow down the range of solutions it could already generate. This is encouraging, as it suggests that REDI-trained models are well-prepared, and potentially even better suited, for subsequent performance gains through online RL.

\begin{table*}[htbp]
    \centering
    \caption{\small\textbf{Pass@16} Performance Comparison for Models Initialized from \texttt{Qwen-SFT-1.5B-3ep}.}
    \label{tab:pass16_3ep}
    \resizebox{\textwidth}{!}{
    \footnotesize
    \begin{tabular}{lccccc}
        \toprule
        Model Configuration & MATH-500 & AIME24 & AMC23 & Minerva & OlympiadBench \\
        \midrule
        \rowcolor{SFTcolor} Qwen-SFT-1.5B-3ep (Start)       & 94.6 & 53.3 & 86.7 & 57.7 & 63.1 \\
        \rowcolor{NeutralColor} DPO \small{($\beta=0.001, \text{LR}=2e-7$)} & 94.6 & 63.3 & 86.7 & 58.8 & \textbf{64.7} \\ 
        \rowcolor{NeutralColor} SimPO \small{($\beta=2, \gamma=1, \text{LR}=5e-7$)} & 94.8 & 60.0 & 89.2 & 58.5 & 63.3 \\
        \rowcolor{NeutralColor} Symmetric REDI \small{($\alpha=1.0, \text{LR}=2e-7$)} & \textbf{95.0} & 63.3 & \textbf{90.4} & 58.8 & 62.5 \\ 
        \rowcolor{REDIcolor} \textbf{REDI} \small{($\alpha=0.8, \text{LR}=1e-6$)} & \textbf{95.0} & \textbf{66.7} & 81.9 & \textbf{59.9} & 63.9 \\
        \bottomrule
    \end{tabular}
    }
\end{table*}
\begin{table*}[htbp]
    \centering
    \caption{\small\textbf{Pass@16} Performance for REDI Initialized from \texttt{Qwen-SFT-1.5B-5ep}.}
    \label{tab:pass16_5ep}
    \resizebox{\textwidth}{!}{
    \footnotesize
    \begin{tabular}{lccccc}
        \toprule
        Model & MATH-500 & AIME24 & AMC23 & Minerva & OlympiadBench \\
        \midrule
        \rowcolor{NeutralColor} DeepSeek-R1-Distill-Qwen-1.5B & \textbf{95.6} & 63.3 & \textbf{92.8} & 56.6 & \textbf{65.8} \\
        \rowcolor{SFTcolor} Qwen-SFT-1.5B-5ep \small{(SFT baseline)} & \textbf{95.6} & 56.7 & 86.7 & 56.3 & 64.3 \\
        \rowcolor{REDIcolor} \textbf{Qwen-REDI-1.5B} \small{($\alpha=0.8, \text{LR}=1e-6$)} & 95.0 & \textbf{66.7} & 90.4 & \textbf{57.0} & \textbf{65.8} \\
        \bottomrule
    \end{tabular}
    }
    \end{table*}

\section{Qualitative Analysis of Model Behavior}\label{sec:qualitative_analysis}

\subsection{Generation Statistics}
\begin{figure*}[ht]
    \centering
    \includegraphics[width=1.0\textwidth]{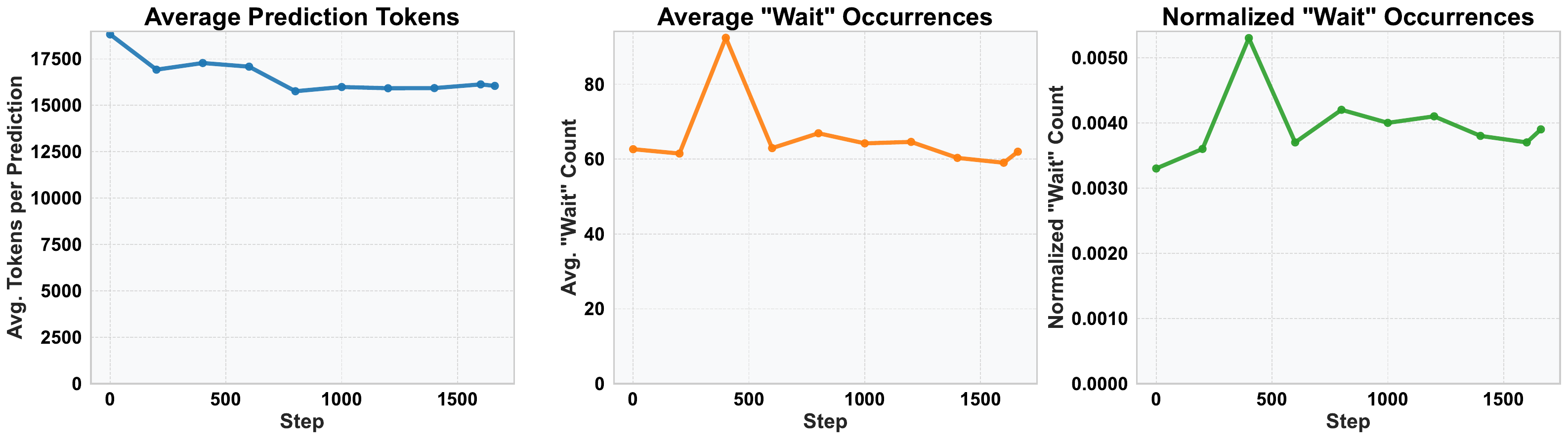}
    \caption{Generation statistics for the \texttt{REDI} model presented in Table~\ref{tab:pass1_3ep} (initialized from \texttt{Qwen-SFT-1.5B-3ep}, trained with REDI using $\alpha=0.8, \text{LR}=1 \times 10^{-6}$) on the AIME24 test set. Metrics shown are (Left) Average Prediction Tokens, (Center) Average ``Wait" Occurrences, and (Right) Normalized ``Wait" Occurrences, all plotted against REDI training steps. Step 0 represents the SFT model before REDI training.}
    \label{fig:redi_generation_stats_appendix}
\end{figure*}

To investigate qualitative changes in reasoning style during REDI training, we monitored key generation statistics. The frequency of terms like ``Wait" serves as an indicator of explicit reflective steps within the model's Chain-of-Thought (CoT), a common trait in reasoning models. The average generation length (token count per prediction) is also crucial, as complex reasoning often correlates with longer outputs. These metrics help assess how REDI influences the model's reasoning trace characteristics.

Figure~\ref{fig:redi_generation_stats_appendix} illustrates these trends for the REDI configuration detailed in Table~\ref{tab:pass1_3ep}. The average token count (left panel) shows a slight decrease from the SFT model's baseline (Step 0), stabilizing at a somewhat lower level during REDI training.

The frequency of ``Wait" occurrences (center panel) exhibits a notable dynamic: a transient increase during early-to-mid REDI training, followed by a return to levels largely comparable with the SFT baseline. This pattern is mirrored by the normalized ``Wait" occurrences (right panel), which accounts for token length.

Overall, while REDI training leads to a modest reduction in average generation length, the model's propensity for explicit reflection, as indicated by ``Wait" counts, shows temporary fluctuations before largely realigning with the SFT base model's characteristics after the initial tuning phase.

\subsection{Case Study on Model Responses}

We examine model responses to AIME 2024 II Problem 3. The generation parameters were consistent with those used in our main evaluations. The problem is presented to the model as follows:

\begin{promptbox}{Problem Statement: AIME 2024 II Problem 3}
Solve the following math problem efficiently and clearly. The last line of your response should be of the following format: `Therefore, the final answer is: \$\textbackslash boxed\{\{ANSWER\}\}\$. I hope it is correct' (without quotes) where ANSWER is just the final number or expression that solves the problem. Think step by step before answering.\\\\
Find the number of ways to place a digit in each cell of a 2x3 grid so that the sum of the two numbers formed by reading left to right is $999$, and the sum of the three numbers formed by reading top to bottom is $99$. The grid below is an example of such an arrangement because $8+991=999$ and $9+9+81=99$.
\[\begin{array}{|c|c|c|} \hline 0 & 0 & 8 \\ \hline 9 & 9 & 1 \\ \hline \end{array}\]
\end{promptbox}
The ground truth answer is 45.

\subsubsection{Response from \texttt{Qwen-SFT-1.5B-3ep} (SFT Baseline)}
This model answered 0/4 attempts correctly for this problem. A representative incorrect response is:
\begin{tcblisting}{title=SFT Baseline Response (Incorrect),
                   breakable,
                   listing only,
                   listing options={style=mystyle}}
<think>
Okay, so I need to figure out how many ways there are to fill a 2x3 grid...
Let me denote the grid as follows:
Row 1: a b c
Row 2: d e f

The total sum of the row numbers is abc + def = 999.
(100a + 10b + c) + (100d + 10e + f) = 999
Which simplifies to:
100(a + d) + 10(b + e) + (c + f) = 999

The total sum of the column numbers is ad + be + cf = 99.
10a + d + 10b + e + 10c + f = 99
Which simplifies to:
10(a + b + c) + (d + e + f) = 99

Let me denote S1 = a + d, S2 = b + e, S3 = c + f
Then equation 1 becomes:
100*S1 + 10*S2 + S3 = 999

Equation 2 becomes:
10*(a + b + c) + (d + e + f) = 99
But a + b + c + d + e + f = S1 + S2 + S3
Therefore, substituting into equation 2:
10*(S1 + S2 + S3 - S3) + S3 = 99 [Error in reasoning: S1+S2+S3-S3 = S1+S2, not a+b+c]
Which simplifies to:
10*(S1 + S2) + S3 = 99

So equation 2 is 10*(S1 + S2) + S3 = 99
Equation 1 is 100*S1 + 10*S2 + S3 = 999

Let me write equation 2 as 10*S1 + 10*S2 + S3 = 99
Then subtract equation 2 from equation 1:
(100*S1 + 10*S2 + S3) - (10*S1 + 10*S2 + S3) = 999 - 99
90*S1 = 900 ⇒ S1 = 10

Then substitute back into equation 2: 10*10 + 10*S2 + S3 = 99
⇒ 100 + 10*S2 + S3 = 99 ⇒ 10*S2 + S3 = -1

But S2 and S3 are sums of digits, so they must be non-negative.
Therefore, 10*S2 + S3 = -1 is impossible. Therefore, there are no solutions.
[...]
Wait, let me check the example again. The example is:
Row 1: 0 0 8 → 8
Row 2: 9 9 1 → 991
Sum: 8 + 991 = 999
Columns:
First column: 0 9 → 09
Second column: 0 9 → 09
Third column: 8 1 → 81
Sum: 09 + 09 + 81 = 99
Yes, this is correct. So the example is a valid solution.
[...]
In the example, S1 = a + d = 0 + 9 = 9, S2 = b + e = 0 + 9 = 9, S3 = c + f = 8 + 1 = 9.
Then equation 1: 100*9 + 10*9 + 9 = 999. Correct.
Equation 2 [Model A's version]: 10*(S1 + S2) + S3 = 10*(9 + 9) + 9 = 189. Which is not equal to 99.
So equation 2 is not satisfied. Therefore, the equations are inconsistent.
[...]
Therefore, according to the equations, there are no solutions. Therefore, the number of ways is zero.

**Final Answer**
The number of ways is \boxed{0}.
</think>
[...]
Thus, the number of ways to place the digits in the grid is \(\boxed{0}\).
\end{tcblisting}
\xpar{Analysis of SFT Baseline Response:}
The SFT model correctly set up the initial two main equations based on row and column sums. However, it made a critical algebraic error when attempting to rewrite the second equation (sum of column numbers) in terms of intermediate sums \(S_1=a+d, S_2=b+e, S_3=c+f\). Specifically, it incorrectly assumed \(a+b+c = S_1+S_2\). This flawed transformation led to an inconsistent system of equations, resulting in the conclusion that \(10S_2+S_3=-1\), which is impossible for sums of digits. Consequently, it incorrectly answered 0. Even when checking the provided example, which contradicted its derived system, the model failed to identify its algebraic mistake and instead reinforced its belief in the inconsistency.
\subsubsection{Response from \texttt{Qwen-SFT-1.5B-3ep + REDI} ($\alpha=0.8$, Checkpoint 1400)}
This model, taken from an intermediate stage of REDI training (step 1400 out of 1661 total steps for 1 epoch), answered 2/4 attempts correctly for this problem. A representative correct response is:
\begin{tcblisting}{title=REDI-tuned Response (Correct),
                   breakable,
                   listing only,
                   listing options={style=mystyle}}
<think>
Okay, so I need to figure out how many ways there are to fill a 2x3 grid...
Let the grid be:
  a b c
  d e f
Row 1 sum: R1 = 100a + 10b + c
Row 2 sum: R2 = 100d + 10e + f
Column 1 sum: C1 = 10a + d
Column 2 sum: C2 = 10b + e
Column 3 sum: C3 = 10c + f
Given that R1 + R2 = 999 and C1 + C2 + C3 = 99.
So, first equation: 100(a + d) + 10(b + e) + (c + f) = 999
Second equation: (10a + d) + (10b + e) + (10c + f) = 99
[...]
Let X = a + d, Y = b + e, Z = c + f.
Row sum equation: 100X + 10Y + Z = 999
Column sum equation [initial incorrect attempt by the model]: 10X + Y + Z = 99
Subtracting the column sum equation from the row sum equation:
90X + 9Y = 900
Divide by 9: 10X + Y = 100
Then, substituting back into column sum equation:
100 + Z = 99 => Z = -1
Which is impossible. Therefore, there is a contradiction, meaning that there are
no solutions? But the problem states that there are such arrangements, so I must
have made a wrong assumption.
[...]
Wait, let's think about the addition of the two three-digit numbers R1 and R2 resulting in 999.
100(a + d) + 10(b + e) + (c + f) = 999
So, 100X + 10Y + Z = 999, where X = a + d, Y = b + e, Z = c + f.
[...] So X can be at most 9. [...] So X must be at least 9. Therefore, X=9.
Therefore, X=9, so 100*9 + 10Y + Z = 999 => 900 + 10Y + Z = 999 => 10Y + Z = 99.
Therefore, Y=9, Z=9. So:
X = a + d = 9
Y = b + e = 9
Z = c + f = 9
Therefore, S = a + b + c + d + e + f = (a + d) + (b + e) + (c + f) = X + Y + Z = 9 + 9 + 9 = 27
[...]
But also, C1 + C2 + C3 = 10(a + b + c) + (d + e + f) = 99
But a + b + c = 27 - (d + e + f). Let me denote D = d + e + f.
Therefore, C1 + C2 + C3 = 10*(27 - D) + D = 270 - 10D + D = 270 - 9D
But we know that C1 + C2 + C3 = 99. Therefore:
270 - 9D = 99 => -9D = -171 => D = 19
Therefore, d + e + f = 19
[...]
So, we have the following equations:
1. a + d = 9
2. b + e = 9
3. c + f = 9
4. d + e + f = 19
5. a + b + c = 8 (since 27-19=8)
6. All variables a, b, c, d, e, f are digits (0-9)
[...]
The number of non-negative integer solutions to d + e + f = 19 with d, e, f <= 9.
Using inclusion-exclusion principle.
[...]
First term: C(19 + 3 -1, 3 -1) = C(21, 2) = 210
Subtract the cases where at least one variable >=10. [...] 3*55=165.
Now, add back the cases where two variables >=10. [...] No solutions.
Therefore, total solutions=210 -165 +0 -0=45.
Therefore, 45 solutions. So there are 45 possible triples (d, e, f) that sum to 19.
But each such triple corresponds to a unique (a, b, c) since a=9 - d, b=9 - e, c=9 - f.
Therefore, the number of ways is 45.

But let's confirm with an example...

Therefore, the number of ways is 45. Therefore, the final answer is \boxed{45}.
**Final Answer**
\boxed{45}
</think>
[...]
Thus, the number of ways to place the digits in the grid is:
\[
\boxed{45}
\]
\end{tcblisting}

\begin{figure*}[ht]
    \centering
    \includegraphics[width=0.5\linewidth]{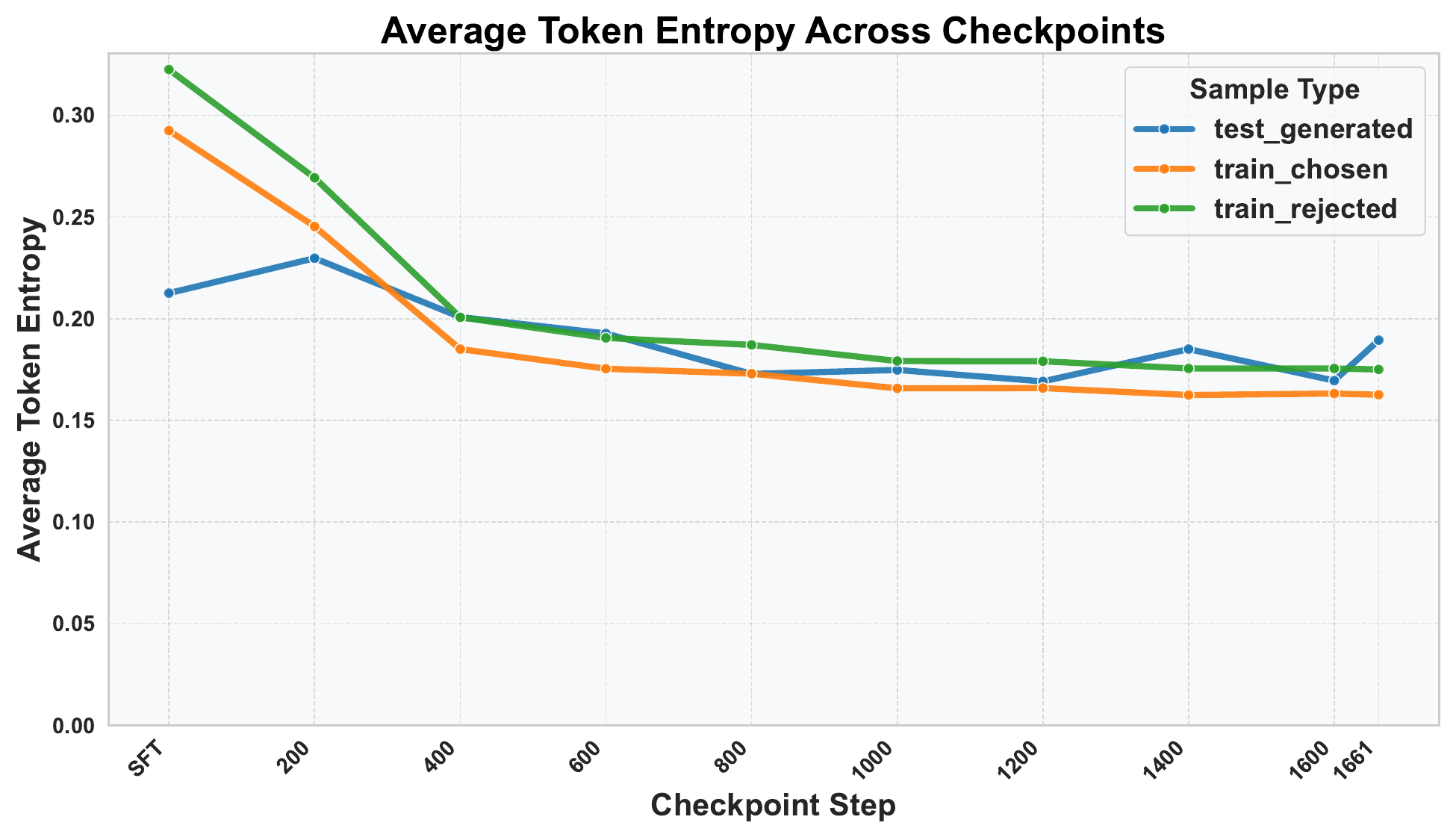}
    \caption{Average token entropy dynamics across training checkpoints for different sample types, for the REDI $\alpha=0.8$ run. ``SFT" denotes the initial model state (\texttt{Qwen-SFT-1.5B-3ep}). ``train\_chosen" and ``train\_rejected" refer to sequences from the preference dataset, while ``test\_generated" refers to sequences generated by the model on test prompts.}
    \label{fig:appendix_entropy_dynamics}
\end{figure*}

\begin{figure*}[ht]
    \centering
    \includegraphics[width=0.65\linewidth]{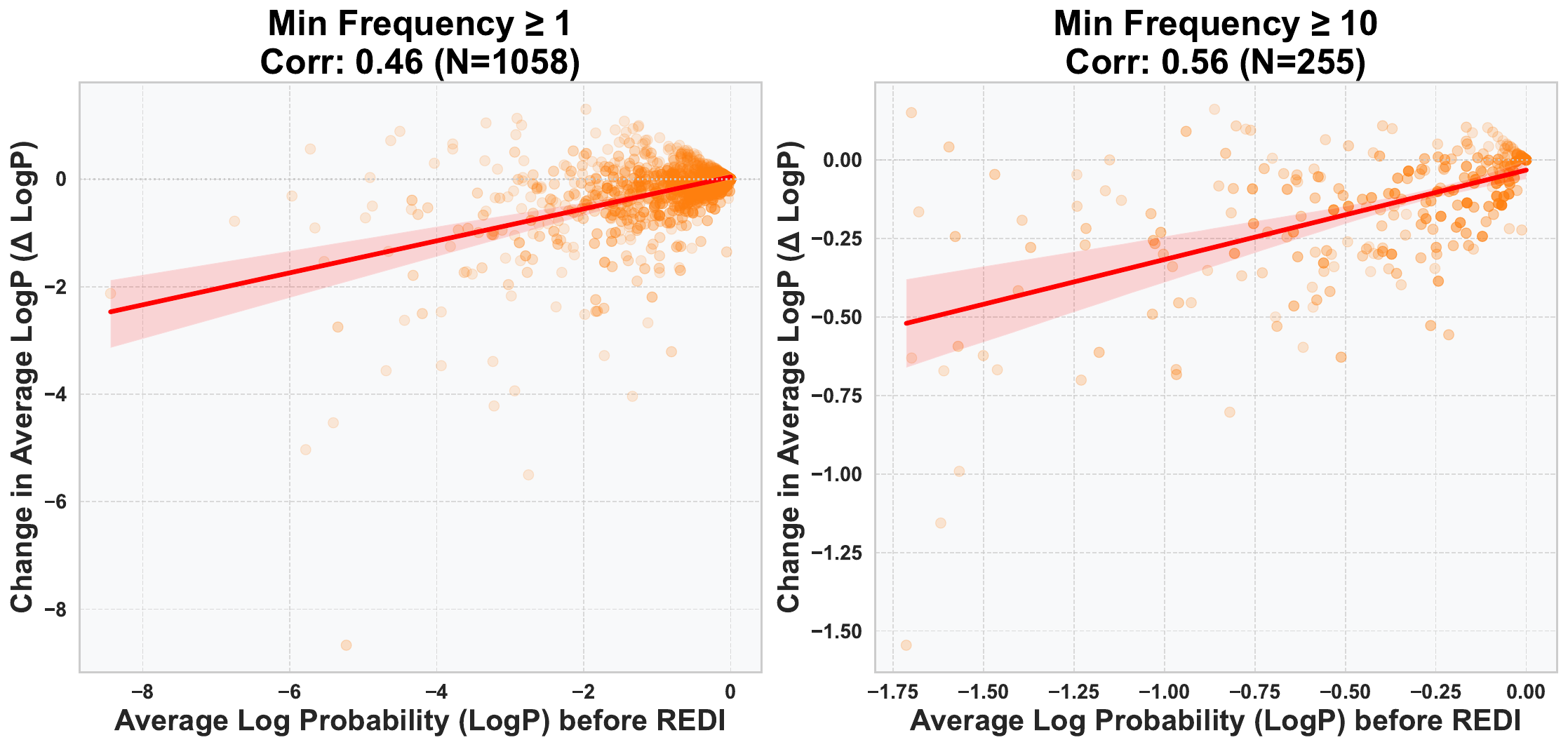}
    \caption{Correlation between the change in average token log-probability (LogP) after REDI training ($\Delta$ LogP) and the average token LogP before REDI training (i.e., in the \texttt{Qwen-SFT-1.5B-3ep} model). Analysis performed on tokens from chosen responses in the training data for the REDI $\alpha=0.8$ run. Left: Tokens with a minimum frequency $\ge 1$. Right: Tokens with a minimum frequency $\ge 10$. A positive correlation indicates that tokens with initially higher LogP tend to see their LogP increase (or decrease less), while tokens with initially lower LogP tend to experience a larger decrease.}
    \label{fig:appendix_logp_correlation}
\end{figure*}

\xpar{Analysis of REDI-tuned Response:}
The REDI-tuned model also correctly established the initial equations. Notably, it initially made a similar algebraic error when attempting to use intermediate sums, leading to a momentary contradiction (\(Z=-1\)). However, unlike the SFT baseline, this model demonstrated an enhanced ability to self-correct or find an alternative path. It revisited the first primary equation (\(100(a+d) + 10(b+e) + (c+f) = 999\)) and correctly deduced from the properties of digit sums that \(a+d=9, b+e=9, c+f=9\). This crucial insight allowed it to determine the total sum of all digits (\(27\)). Using this, along with the second primary equation (\(10(a+b+c) + (d+e+f) = 99\)), it correctly solved for the sum of digits in the top row (\(a+b+c=8\)) and bottom row (\(d+e+f=19\)). Finally, it correctly applied the Principle of Inclusion-Exclusion to count the number of ways to form \(d+e+f=19\) with digits, leading to the correct answer of 45.

This case study suggests that REDI training, by incorporating signals from both positive and negative reasoning traces, can enhance a model's ability to navigate complex problem-solving paths, including recovering from intermediate errors and identifying correct solution strategies, which might be less developed in models trained solely on positive examples.

\section{Additional Analysis on Training Dynamics}
\label{sec:appendix_training_dynamics}

To further investigate the training dynamics of REDI, particularly the observed phenomenon where both chosen ($y_w$) and rejected ($y_l$) log-probabilities (LogPS) can decline even in stable runs, we tracked additional statistics. The specific run analyzed here is our REDI configuration with $\alpha=0.8$, initialized from the \texttt{Qwen-SFT-1.5B-3ep}. We randomly sampled 4 prompts from the preference dataset ($\mathcal{D}_{\text{Pref}}$), supplemented by 1 test prompt from AIME24, 1 from AIME25, and 2 from MATH-500. For the training data samples, we analyzed the logits directly. For the test data samples, we first performed auto-regressive generation and then analyzed the logits of the generated sequences.

Figure~\ref{fig:appendix_entropy_dynamics} illustrates the average token entropy across training checkpoints for the REDI $\alpha=0.8$ run. We observe a rapid decrease in entropy from the initial SFT model state to approximately step 400-600 of REDI training. This decrease is more pronounced for sequences from the training data (``train\_chosen" and ``train\_rejected") compared to sequences generated by the model on test prompts (``test\_generated"). This suggests that the model becomes more confident (i.e., assigns sharper probability distributions) over tokens when conditioned on training sequences. After the initial drop, the entropy tends to stabilize or fluctuate slightly.

To understand where the model's probability mass is shifting, we analyzed the change in token log-probabilities (LogPs) relative to their initial values in the \texttt{Qwen-SFT-1.5B-3ep} model. Figure~\ref{fig:appendix_logp_correlation} displays this relationship for tokens in the chosen responses from the training data, specifically for the REDI $\alpha=0.8$ run. A moderate positive correlation is observed (Pearson correlation coefficient of 0.46 for tokens with frequency $\ge 1$, and 0.56 for tokens with frequency $\ge 10$). This positive correlation suggests that, on average:
\begin{itemize}[topsep=1pt,itemsep=2pt,leftmargin=20pt]
    \item Tokens for which the SFT model already had a high probability (higher initial LogP) tend to see their probabilities further increase or decrease less after REDI training.
    \item Conversely, tokens for which the SFT model had a low probability (lower initial LogP) tend to experience a more significant decrease in their probabilities.
\end{itemize}
In essence, REDI appears to amplify the model's existing tendencies to some extent, making it more confident about tokens it was already likely to predict and even less confident about tokens it was unlikely to predict. This behavior, where negative gradients might inadvertently suppress probabilities of tokens beyond the specific negative example, has been discussed in recent literature on off-policy preference optimization \citep{yan20253dpropertiesidentifyingchallengesdpo, ren2025learning, razin2025unintentional}.

Despite these complex dynamics and the observed shifts in token probabilities, our main results in Section~\ref{sec:results} indicate that as long as the training process avoids catastrophic collapse (which REDI's asymmetric weighting helps to prevent), the model achieves strong performance improvements on downstream reasoning tasks.

The analyses presented in this section are preliminary and offer initial insights into REDI's learning mechanisms for the specific $\alpha=0.8$ configuration. A more comprehensive understanding of how REDI precisely refines the model's internal representations and generation strategies warrants further investigation.
\section{Licenses for Models, Data, and Benchmarks}
\label{sec:licenses}

\paragraph{Scope.} This section summarizes the licenses for (i) base models we fine-tuned or referenced, and (ii) datasets/benchmarks we trained on or evaluated against. It is not legal advice; please review the upstream licenses directly before redistribution or commercial use.

\subsection*{Models}

\noindent\textbf{Qwen2.5-Math-1.5B.}
Official releases in the Qwen2.5-Math series (including the 1.5B variant) are distributed under the \emph{Apache License 2.0}.\footnote{\url{https://huggingface.co/Qwen/Qwen2.5-Math-1.5B}\;(\texttt{License: apache-2.0})}%
\footnote{Qwen team’s overview noting most Qwen2.5 models are Apache 2.0, with exceptions for some sizes: \url{https://qwenlm.github.io/blog/qwen2.5-llm/}.}

\noindent\textbf{Llama~3.2 (for comparisons/references).}
Meta’s Llama~3.2 family is covered by the \emph{Meta Llama 3.2 Community License}. It allows broad research and many commercial uses subject to the license terms (e.g., usage restrictions and acceptable use policy).\footnote{\url{https://ai.meta.com/resources/models-and-libraries/llama-downloads/}}\footnote{Model card mirror: \url{https://huggingface.co/meta-llama/Llama-3.2-1B}, see ``License'' section.}

\subsection*{Distilled / Training Data}

\noindent\textbf{OpenR1-Math-Raw (teacher-distilled traces).}
The \textit{OpenR1-Math-Raw} corpus we derive our SFT and preference pairs from is released under \emph{Apache License 2.0}.\footnote{\url{https://huggingface.co/datasets/QwenLM/OpenR1-Math-Raw}\;(\texttt{License: apache-2.0}).}

\subsection*{Evaluation Benchmarks}

\noindent\textbf{MATH and MATH-500.}
The original MATH dataset repository does not specify an explicit open-source license; treat problem text as copyrighted competition material and use under research/academic terms only.\footnote{\url{https://github.com/hendrycks/math}}
Some third-party MATH-500 mirrors label \texttt{apache-2.0}, but these are not the canonical license from the original authors; verify before redistribution.\footnote{\url{https://www.modelscope.cn/datasets/AI-ModelScope/MATH-500}.}

\noindent\textbf{AIME24 and AMC23 (MAA).}
AIME/AMC problems are copyrighted by the Mathematical Association of America (MAA). Public mirrors (e.g., AoPS) note the MAA copyright explicitly. Use for research/academic evaluation typically falls under fair use or with permission; redistribution may require permission from MAA.\footnote{MAA AMC overview: \url{https://maa.org/student-programs/amc/}.}\footnote{AoPS page with copyright notice: \url{https://artofproblemsolving.com/wiki/index.php/2024_AIME_I_Problems}.}

\noindent\textbf{Minerva (evaluation set).}
The Minerva paper reports results on collections of STEM problems compiled by Google; there is no separately packaged ``Minerva dataset'' with a permissive license in the original paper materials. If you reproduce those evaluations, ensure the source problems’ terms permit such use, and limit redistribution of copyrighted items.\footnote{Paper: \url{https://arxiv.org/abs/2206.14858}.}\footnote{Google blog: \url{https://research.google/blog/minerva-solving-quantitative-reasoning-problems-with-language-models/}.}

\noindent\textbf{OlympiadBench.}
The official OlympiadBench repository is released under the \emph{MIT License}. Note that it aggregates problems from Olympiad and exam sources; the repo’s code/data package is MIT-licensed, but underlying problem texts and images may carry third-party copyrights—follow any per-file notices.\footnote{GitHub (shows ``MIT license''): \url{https://github.com/OpenBMB/OlympiadBench}.}

\section{Statement on the Use of AI Assistants} AI assistants were utilized during the preparation of this manuscript for proofreading. The scope of their use was limited to improving grammar, style, and clarity. All core content, analysis, and conclusions presented in this work are the original contributions of the authors.

\section{Potential Harms}
The development of more powerful and efficient reasoning models, as presented in this work, has potential for misuse. Our method makes smaller, open-source models more capable, which lowers the barrier for malicious applications such as generating sophisticated disinformation, finding software vulnerabilities, or automating social engineering attacks. This accessibility bypasses the safeguards common to larger, API-gated models.

Furthermore, the distillation process can propagate subtle biases from the teacher model, and improved benchmark scores may mask underlying reliability issues. This poses significant risks if these models are deployed in high-stakes applications without rigorous human oversight. We advocate for continued research into AI safety and robust evaluation to ensure such technologies are developed and deployed responsibly.

\end{document}